\documentclass[ejs]{imsart}

\RequirePackage[OT1]{fontenc}
\usepackage{amsthm,amsmath,amssymb}
\usepackage{bm,bbm}
\usepackage{graphicx}
\RequirePackage[round]{natbib}
\RequirePackage[colorlinks,citecolor=blue,urlcolor=blue]{hyperref}
\citefix
\usepackage{subcaption}
\usepackage{color}
\usepackage{enumitem}
\usepackage{comment}

\usepackage{algorithm}
\usepackage{algpseudocode}

\pubyear{2020}
\volume{0}
\issue{0}
\firstpage{1}
\lastpage{8}
\arxiv{2010.00000}

\startlocaldefs
\theoremstyle{plain}
\newtheorem{theorem}{Theorem}
\newtheorem{corollary}{Corollary}%[theorem]
\newtheorem{lemma}{Lemma}
\newtheorem{remark}{Remark}

\newtheorem{definition}{Definition}
\newtheorem{assumption}{Assumption}
\newtheorem{proposition}{Proposition}

\def\supp{\text{supp}}
\endlocaldefs

\begin{document}

\begin{frontmatter}

% "Title of the Paper"
% \title{???\thanksref{t1}}
% \thankstext{t1}{This is an original survey paper}
% \runtitle{???}

\title{Regret Lower Bound and Optimal Algorithm for High-Dimensional Contextual Linear Bandit}
\runtitle{High-Dimensional Contextual Linear Bandit}

% indicate corresponding author with \corref{}
% \author{\fnms{John} \snm{Smith}\thanksref{t2}\corref{}\ead[label=e1]{smith@foo.com}\ead[label=e2,url]{www.foo.com}}
% \thankstext{t2}{Thanks to somebody} 
% \address{line 1\\ line 2\\ \printead{e1}\\ \printead{e2}}

\begin{aug}

\author{\fnms{Ke} \snm{Li}\ead[label=e1]{kel6@illinois.edu}}
\address{Department of Statistics, University of Illinois at Urbana-Champaign \\
\printead{e1}}
% \and
\author{\fnms{Yun} \snm{Yang}\ead[label=e2]{yy84@illinois.edu}}
\address{Department of Statistics, University of Illinois at Urbana-Champaign \\
\printead{e2}}
\and
\author{\fnms{Naveen N.} \snm{Narisetty}
\ead[label=e3]{naveen@illinois.edu}}
\address{Department of Statistics, University of Illinois at Urbana-Champaign \\
\printead{e3}}

\runauthor{Ke Li et al.}
\affiliation{University of Illinois at Urbana-Champaign}

\end{aug}

\begin{abstract}
In this paper, we consider the multi-armed bandit problem with high-dimensional features. First, we prove a minimax lower bound, $\mathcal{O}\big((\log d)^{\frac{\alpha+1}{2}}T^{\frac{1-\alpha}{2}}+\log T\big)$, for the cumulative regret, in terms of horizon $T$, dimension $d$ and a margin parameter $\alpha\in[0,1]$, which controls the separation between the optimal and the sub-optimal arms. This new lower bound unifies existing regret bound results that have different dependencies on T due to the use of different values of margin parameter $\alpha$ explicitly implied by their assumptions. Second, we propose a simple and computationally efficient algorithm inspired by the general Upper Confidence Bound (UCB) strategy that achieves a regret upper bound matching the lower bound. The proposed algorithm uses a properly centered $\ell_1$-ball as the confidence set in contrast to the commonly used ellipsoid confidence set. In addition, the algorithm does not require any forced sampling step and is thereby adaptive to the practically unknown margin parameter. Simulations and a real data analysis are conducted to compare the proposed method with existing ones in the literature.
\end{abstract}


\tableofcontents
\end{frontmatter}

% Main text entry area

\section{Introduction}

In the big data era, the abundance of personalized information enables decision-makers to make individualized decisions for improving the long term reward by incorporating this contextual information. For example, internet marketing companies may utilize users' searching history and demographics to display personalized online advertisements to improve the conversion rate \citep{abe}. In personalized medicine, doctors assign treatments tailored to the individual patient based on the context of the patient's own medical records and genetic information such as biomarkers \citep{bastani15}. In these examples, data are often collected sequentially and decision-makers need to pick the best action to maximize the long term reward based on the current predict response in a sequential fashion. Mathematically, this problem can be formulated as a contextual bandit problem \citep{abe, chu}, where an agent sees a $d$-dimensional feature vector and has to choose among $K$ actions (arms) at each of the $T$ rounds to maximize the cumulative reward.
When the expected reward is a linear function of the features, this problem is known as the linear bandit problem, or the multi-armed bandit problem with linear payoff functions \citep{abe, auer}. Under this setting, \citet{dani}, \citet{chu} and \citet{abbasi11} showed a polynomial dependence of the cumulative regret on dimension $d$ and time horizon $T$. Precisely, \citet{chu} proved a regret upper bound scaling as $\mathcal{O}(\sqrt{dT})$, while \citet{dani} and \citet{abbasi11} proved regret upper bounds scaling as $\mathcal{O}(d\sqrt{T})$.

We focus on the high-dimensional regime where the dimension $d$ of the feature vector can be comparable, or even much larger than the total number of rounds $T$ that plays the role of ``sample size" in the statistical perspective. The high-dimensionality of the features and the dependence between the samples induced by the bandit policy make the high dimensional linear bandit problem very challenging both for methodological development and theoretical analysis. In particular, the ordinary least squares (OLS), a traditional statistical method for linear regression and its variants serve as the cornerstones for most linear bandit algorithms, they require a substantial number of samples to accurately estimate the parameters, incurring the polynomial-$d$ dependence of the regret. In the statistical literature, it is well-known that imposing extra lower-dimensional structures such as sparsity on the model improves the minimax risk, or the sample complexity of the best learning algorithm. For example, under the sparsity assumption that the response, or reward, only depends on a small subset of all features with size $s_0\ll d$, the minimax risk reduces from $\sqrt{d/T}$ to $\sqrt{s_0\log d/T}$, where $s_0$ now plays the role of the ``effective dimension" and the extra $\log d$ term is due to the uncertainty in feature selection. Unfortunately, sparsity generally does not help improve the regret in linear bandit. Roughly speaking, this failure of sparsity adaptation is due to the exploitation nature of a good bandit problem that tends to choose an optimal arm as often as possible, which prevents the exploration of different directions in the feature space and results in an ill-conditioned design matrix. Interestingly, we find that sparsity adaptation is still possible in the high-dimensional linear bandit when features associated with different arms have a sufficient separation. In particular, we show that under the stochastic assumption that feature vectors are random whose population-level covariance matrices are well-conditioned (see Assumption \ref{Assump:stochastic} for a precise definition), the randomness in the features prevents the selected arm to collapse and leads to a regret bound that scales logarithmically in the feature dimension $d$. Moreover, this randomness in the feature implicitly encourages explorations in the feature space, due to which no \emph{forced exploration} is needed. A simple and computationally efficient algorithm combining the Upper Confidence Bound (UCB) technique \citep{auer02a, auer} with the LASSO \citep{lasso} estimator induces sparsity in the estimated regression parameter and enjoys the improved regret bound.

In this paper, we make a number of contributions to the multi-armed bandit literature. Theoretically, we introduce a relaxation of the margin condition (Assumption \ref{Assump:reward}\ref{Assump:margin}) from \citet{wang18} and \citet{bastani15}, which precisely captures the hardness of the linear bandit problem via a margin parameter $\alpha \in [0,+\infty]$ that controls the ``separation" between optimal and sub-optimal arms. With this additional assumption, we prove a regret lower bound of $\mathcal{O}(\log d+\log T)$ for $\alpha = 1$ and $\mathcal{O}((\log d)^{\frac{\alpha+1}{2}} T^{\frac{1-\alpha}{2}})$ for $\alpha \in [0,1)$. The dependence of our lower bound on the margin parameter $\alpha$ unifies both the polynomial and the logarithmic regret bound dependencies on $T$ in the literature. More details are provided in the discussion of related works in Section \ref{Sec:literature}. Methodologically, we propose the use of a properly centered $\ell_1$-ball as the confidence set in contrast to the commonly used ellipsoid confidence sets in the UCB algorithm. We show that the $\ell_1$-ball confidence set captures the uncertainty and implicitly manages the trade-off between exploration versus exploitation. We also prove that the algorithm is optimal up to a logarithmic factor in $T$.

Technically, our proof of the upper bound is based on non-asymptotic analysis for the LASSO estimator for dependent data based on a novel proof strategy for verifying the \textit{restricted eigenvalue condition} for the sample covariance matrix under dependent data. In particular, our analysis uses an anti-concentration technique to show that the randomness in the feature vectors facilitates an automatic feature space exploration and prevents feature collapse. Computationally, optimizing the expected reward jointly over the $\ell_1$-ball confidence set and the covariate set is equivalent to finding the arm that maximizes the estimated reward plus a term proportional to the $\ell_\infty$-norm (dual norm to the $\ell_1$-norm) of the corresponding feature, which is easy to implement and computationally efficient. 

The rest of the paper is organized as follows. In Section \ref{Sec:formulation}, we review the background and discuss our assumptions. After that, we present the regret lower bound for policies. In section \ref{Sec:algorithm}, we propose $\ell_1$-confidence ball based algorithm for the high-dimensional linear bandit problem. We present a novel non-asymptotic analysis of the LASSO estimator under the bandit setting, and provide an regret upper bound for the proposed algorithm. Section \ref{Sec:experiment} reports experimental results comparing with some existing bandit algorithms. Detailed proofs of theorems are deferred to the appendix section. 

\section{Problem Formulation}\label{Sec:formulation}

Let $T$ be the total time number of steps, which is allowed to be unknown, and $K$ be the number of possible actions (arms). At each round $t=1,\ldots,T$, for each arm $a\in[K]$, the learner observes one of $K$ feature vectors $X_{a,t}\in\mathbb R^d$. We consider the high-dimensional regime where the dimension $d$ of the feature space can be comparable or even much larger than the time horizon $T$. We adopt a standard random-design framework by assuming that for each arm $a\in [K]$, the observed sequence of feature vectors $\{X_{a,t}\}_{t\ge 0}$ are i.i.d. sub-Gaussian random vectors drawn from an unknown distribution $\mathbb{P}_{a}$ in $\mathbb{R}^d$. The distributions $\mathbb{P}_a$'s for $a\in[K]$ are allowed to have dependence across different arms. More precisely, recall the following notion of sub-Gaussian random variables. 
\begin{definition}[Sub-Gaussian random variable] \label{Def:1}
 A real-valued random variable $z$ is $\sigma$-sub-Gaussian if $\mathbb{E}[e^{tz}]\le e^{\sigma^2t^2/2}$ for every $t\in \mathbb{R}$. This definition implies $\mathbb{E}[z]=0$ and $Var[z]\le \sigma^2$.
\end{definition}

Let $\mathcal{D}_t=\{X_{a,t}\}_{a=1}^K$ denote the covariate set consisting of feature vectors corresponding to all the arms at time $t\in[T]$. The feature vectors are allowed to be dependent across arms but are assumed to be independent across different time points $t$. The decision-maker has access to $K$ arms and each arm yields a random reward, the expected value of which follows a linear function of the associated feature. All arms share the same unknown parameter $\beta_\ast\in \mathbb{R}^d$ in the linear reward function $\langle \,\cdot\,,\,\beta_\ast\rangle$.
In our setup, if the learner pulls arm $a\in [K]$ at time $t$, then the following random reward is observed $Y_{a,t} :=\langle X_{a,t},\beta_\ast\rangle +\epsilon_{a,t}$, where $\epsilon_{a,t}$ are independent $\sigma$-sub-Gaussian random variables that are also independent of the sequence $\{X_{a,t}\}_{t\ge0}$ for all $a\in [K]$. Define $\sigma$-algebra $\mathcal{F}_t=\sigma(X_{a,1},\epsilon_{a,1},\dots, X_{a,t-1},\epsilon_{a,t-1}, X_{a,t}: \forall a\in[K])$. Then the sequence of $\{\mathcal{F}_t\}_{t=0}^{\infty}$ is a \textit{filtration}, and the errors  $\{\epsilon_{a,t}\}_{t=1}^{\infty}$ for each arm $a\in[K]$ forms a \textit{martingale difference sequence} relative to this filtration.

Our goal is to design a sequential decision-making policy $\pi$ that learns the parameter $\beta_\ast$ over time in order to maximize the expected cumulative reward over the time horizon. Let $\hat{a}_t\in [K]$ denote the arm chosen by the policy $\pi$ at time $t\in [T]$. Then, an admissible policy $\pi$ is a sequence of random variables $\hat{a}_1, \hat{a}_2,\cdots$ taking values in the set $\{1, \cdots, K\}$ such that $\hat{a}_t$ is measurable with respect to the $\sigma$-algebra $\mathcal{F}^{+}_{t}$ generated by the previous feature vectors from each arm and observed rewards of the chosen arms $\{X_{a,s}, Y_{\hat{a}_s, s}: s=1,\cdots, t-1; a=1,\cdots, K\}$ and by the current feature vectors $\{X_{1,t},\cdots, X_{K,t}\}$,
\begin{align*}
    \mathcal{F}^{+}_{t} = \sigma\left(X_{a,1}, Y_{\hat{a}_1, 1}, \cdots, X_{a,t-1}, Y_{\hat{a}_{t-1}, t-1}, X_{a,t}: \forall a\in[K]\right).
\end{align*}
To characterize the quality of the policy $\pi$, we compare it with the oracle policy $\pi^{\ast}$ that uses the knowledge of $\beta_\ast$ to choose the best arm $a_t^{\ast}=\arg \max_a \langle X_{a,t}, \beta_\ast\rangle$ maximizing the expected reward at each round $t\in[T]$. Note that under our random-design assumption, $a_t^\ast$ is also random.  To summarize, any policy incurs at each time $t\in[T]$ an expected regret 
\begin{equation}\label{eq:1}
    r_t=\mathbb{E}\left[\max_{a\in[K]} \langle X_{a,t}, \beta_\ast\rangle-\langle X_{\hat{a}_t,t},\beta_\ast\rangle\right],
\end{equation}
where the expectation is taken with respect to the randomness in the feature vector $\{X_{a,t}\}$ and the stochastic reward through $\{\epsilon_{a,t}\}$.
Our goal is to seek a policy that minimizes the expected cumulative regret $R_T=\sum_{t=1}^T r_t$. 

In the high-dimensional regime, the regression coefficient $\beta_\ast$, which is a high-dimensional parameter vector, is assumed to admit a sparse structure, namely, the number of nonzero entries in $\beta_\ast$ is much smaller than the ambient dimension $d$. We denote $S=\{j: \beta_{\ast,j}\neq 0\}$ as the unknown index set of the nonzero entries of true parameter $\beta_\ast$ that collects all influential feature components. Let $s_0=|S|$ be the number of the nonzero entries in $\beta_\ast$, which satisfies $s_0\ll d$. For each $t\in [T]$, let $X_t\in \mathbb{R}^{t\times d}$ be the time $t$-observed design matrix whose rows correspond to all the selected feature vectors $X_{\hat{a}_s,s}$ for $s=1,\ldots, t$ after $\hat{a}_s$ being pulled, and vector $Y_t=(Y_1,\ldots, Y_t)^T\in\mathbb R^{t}$ collects the observed rewards. 

\subsection{Assumptions}\label{Sec:Assum}

Let $\hat{\Sigma}_t=X_t^TX_t/t$ denote the sample covariance matrix at the end of round $t\in [T]$. For each matrix $A\in\mathbb R^{n\times d}$ and index set $I\subset[d]$, we use $A_I$ to denote the submatrix of $A$ that collects all columns whose indices are in $I$.
Let $\mathbb S^{d\times d}$ denote the cone of all $d\times d$ positive semidefinite matrices. We now define the \emph{sparse eigenvalue condition} for a matrix in $\mathbb S^{d\times d}$.

\begin{definition}[Sparse eigenvalues]\label{Def:sparse_eigenvalue}
 For a $d\times d$ matrix $A\in \mathbb S^{d\times d}$ and $1 \leq m \leq d$, define $\rho_{\min}(m, A)$ and $\rho_{\max}(m, A)$ as the minimum and maximum $m$-sparse eigenvalues of $A$ if 
 \begin{align*}
     \rho_{\min}(m, A) \leq \frac{\Delta^T A \Delta}{\|\Delta\|_2^2} \leq \rho_{\max}(m, A) ~for~any~\Delta \neq 0~with~ \|\Delta\|_0\leq m.
 \end{align*}
\end{definition}
Sparse eigenvalue condition is an important requirement in high-dimensional estimation problems \citep{oliveira, Rudelson}. In particular, it often serves as a bridge for proving the \emph{restricted eigenvalue condition} (c.f.~Definition \ref{Def:restricted_eigenvalue} In Section~\ref{sec:sketched_proof_upper_bound}), which is nearly necessary for accurately learning the unknown parameter $\beta_\ast$ according to high-dimensional statistics literature \citet{highdim}. 

We now state two key assumptions for our theoretical analysis.

\begin{assumption}[On the true reward]\ \label{Assump:reward}
\begin{enumerate}[label={(\alph*)}]
    \item \label{Assump:sparsity} Sparsity condition: There exist positive constants $b$ and $s_0$ such that \\ $\|\beta_\ast\|_1\le b$ and $\|\beta_\ast\|_0\le s_0$. 
    \item \label{Assump:margin} Margin condition: There exist positive constants $\Delta_\ast$, $C_1$ and $\alpha\in[0, + \infty]$, such that for $h\in \left[C_1\sqrt{\frac{\log d}{T}},\Delta_\ast \right]$, $P(\langle X_{a^*_t,t},\beta_\ast\rangle \le \max_{b\neq a^*_t}\langle X_{b,t},\beta_\ast\rangle +h)\le \frac{1}{2}(\frac{h}{\Delta_\ast})^{\alpha}$.
\end{enumerate}
\end{assumption}

The first part of the assumption requires boundedness and sparsity of the true parameter $\beta_\ast$ for its estimability, which are standard requirements in the literature \citep{abbasi11, bastani15}. 

% \textcolor{blue}{We only consider the case when $\alpha\in[0,1]$ for simplicity. In the $\alpha>1$ case, the density of the feature vectors around the margin will be approximately zero and the feature vectors are less likely to fall into a small neighborhood of the decision boundary between the optimal arm and sub-optimal arms, so the bandit problem will be easier than that for the $\alpha=1$ case on selecting the optimal arm. In addition, if the underlying distribution of $X_{a,t}$'s satisfies the margin condition for $\alpha>1$, it also satisfies for $\alpha=1$ case since $h/\Delta_{\ast} \leq 1$ for $h \leq \Delta_{\ast}$.} 
% When the exponent $\alpha$ is close to $0$, the probability mass of the feature vectors falling close to the decision boundary becomes larger, making it harder to distinguish the optimal arm from the sub-optimal arms. 
The second part of the assumption controls the probability of the features falling into an $h$-neighborhood of the decision boundary. As $\alpha$ increases, the margin condition becomes stronger since the feature vectors are less likely to fall close to the decision boundary. As a result, it will be easier for a bandit policy to distinguish the optimal arm from sub-optimal arms. In particular, in the most extreme case with $\alpha=+\infty$, there is a deterministic positive gap between the rewards of the optimal arm and sub-optimal arms, making the bandit problem easiest. On the contrary, the margin condition becomes weaker as $\alpha$ decreases as it is satisfied by a large class of distributions. For example, with $\alpha=0$, the feature vectors can have arbitrary distributions around the decision boundary and do not need to exhibit any separation, making the bandit problem hardest. In particular, our condition under $\alpha=0$ corresponds to the setup considered in \citet{abbasi11}, where at each round $t$, the decision-maker is given an arbitrary set $\mathcal{D}_t$ of feature vectors with no separation assumptions on the feature vectors. As a consequence, the cumulative regret of their algorithm scales $\mathcal{O}(\sqrt{T})$ in horizon $T$.

The margin condition with $\alpha=1$ has been assumed by \citet{golden}, \citet{wang18} and \citet{bastani15} and will be satisfied when the density of  $\langle X_{a,t},\beta_\ast\rangle$ is upper bounded uniformly for $\forall a \in[K]$. 
Assumption~\ref{Assump:reward}\ref{Assump:margin} is a probabilistic relaxation of the usual gap assumption made for problem dependent bounds in the multi-armed bandit literature. For example, \citet{abbasi11} assume that there exists some gap $\Delta>0$ between the rewards of the best arm and the ``second best" arm in the covariate set $\mathcal{D}_t$ such that $ \min_{b\neq a^\ast_t}\langle X_{a^\ast_t,t} - X_{b,t}, \beta_\ast\rangle > \Delta$, which corresponds to our margin condition with $\alpha=+\infty$ and $\Delta_{\ast} = \Delta$. Assumption~\ref{Assump:reward}\ref{Assump:margin} relaxes this by allowing the gap $\langle X_{a^*_t,t},\beta_\ast\rangle-\max_{b\neq a^*_t}\langle X_{b,t},\beta_\ast\rangle$ to be arbitrarily close to zero. 
This assumption resembles the usual margin condition in the classification literature \citet{audibert2007fast}. The probability of features falling into an $h$-neighborhood of the decision boundary, $P(\{\langle X_{a^*_t,t},\beta_\ast\rangle\leq \langle X_{b,t},\beta_\ast\rangle+h,\; \forall b\neq a^*_t\})$, diminishes to zero as $h\to 0_+$, where $a^\ast_t$ is the optimal arm at time $t$. In our regret analysis, we will choose a diminishing  sequence of $h$ values $\{h_t:\,t\in[T]\}$ to control the probability of pulling sub-optimal arms for deriving the regret bound in Theorem \ref{Thm:2}. 
The constant $\Delta_\ast$ in this assumption is related to the effective number of candidates for optimal arms. For example, if only $m$ many arms have strictly positive probability to be optimal, then under the assumption that the probability density function of $\langle X_{a,t},\beta_\ast\rangle$ is bounded for each $a\in[K]$, a simple union argument verifies Assumption~\ref{Assump:margin} with $1/\Delta_\ast=O(m)$.

We provide a concrete example of distributions for which the margin condition holds with different values of $\alpha$. We consider a 2-armed bandit problem with the covariate set $\mathcal{D}_t=\{X_t, X_t + Z_t\} \subseteq \mathbb{R}^d$ in each round $t$, where $X_t$ can follow any distribution and $Z_t=(\zeta_t, 0, \cdots, 0)^T \in \mathbb{R}^d$ has a random variable $\zeta_t$ at the first entry and zero otherwise. In addition, we assume that the first entry of the parameter vector $\beta_{\ast}$ is $1$. The margin condition for the 2-armed bandit problem can be expressed as
\begin{align*}
    P(|\langle Z_t,\beta_\ast\rangle| \leq h) = P(|\zeta_t| \leq h)\leq \frac{1}{2}(\frac{h}{\Delta_\ast})^{\alpha}.
\end{align*}
Then, for any random variable $\zeta_t$ with distribution satisfying the above inequality around $0$, the margin condition will hold with $\alpha$. For example, $\zeta_t$ can be drawn from a random signed Beta distribution, i.e., $\zeta_t\sim \varepsilon\, \mathrm{Beta}(\alpha, 1)$, where $\varepsilon$ equals to $\pm 1$ with equal probabilities. In this example, parameter $\alpha>0$ corresponds to the margin parameter. This 2-armed bandit problem is also applicable in real world applications, in which the competing two candidates are very similar and only a few features has marginal difference between them. In such cases, the margin condition can be viewed as an assumption on the difference, e.g., $Z_t$, between the feature vectors and the parameter $\alpha$ characterizes the level of similarity among different feature vectors.

\begin{assumption}[On the features] \
\label{Assump:stochastic}
\begin{enumerate}[label={(\alph*)}]
    \item \label{Assump:boundedx} Boundedness: For some positive constant $x_{\max}$, $\|X_{a,t}\|_{\infty}\le x_{\max}$ for all $a\in [K]$ and $t\in [T]$.
    \item \label{Assump:anti} Anti-concentration: There exists a positive constant $\zeta$ such that for each $a\in [K]$, $t\in [T]$, $v\in \mathbb R^d$ and $h\in \mathbb{R}_+$, $P(\langle X_{a,t},v\rangle^2 \le h\|v\|_2^2)\le \zeta h$.
    \item \label{Assump:eigenvalue} Sparse eigenvalues: There exist constants $\Lambda_0\ge\lambda_0>0$ and $C^*>0$ such that for each $a\in[K]$, the minimum and maximum $C^\ast s_0$-sparse eigenvalues of the matrix $\Sigma_a :=\mathbb{E}[X_{a,t}X_{a,t}^T]$ are bounded below and above by  $\lambda_0^2$ and $\Lambda_0^2$, respectively. In addition, there exist positive constants $\Lambda_1^2\geq \phi_0^2>0$ such that the minimum and maximum $C^\ast s_0$-sparse eigenvalues of $\mathbb{E}[X_{a^*_t}X_{a^*_t}^T|\Gamma_t]$ are bounded below and above by $\phi_0^2$ and $\Lambda_1^2$, respectively. The event $\Gamma_t$ is defined as $\Gamma_t:=\left\{\langle X_{a^*_t,t},\beta_*\rangle \ge \max_{b\neq a^*_t}\langle X_{b,t},\beta_*\rangle +\Delta_*\right\}$.
\end{enumerate}
\end{assumption}
Assumption~\ref{Assump:stochastic}\ref{Assump:boundedx} together with Assumption~\ref{Assump:reward}\ref{Assump:sparsity} ensures that the maximum regret at each time is bounded, since $|\langle X_{a,t},\beta_\ast\rangle| \le x_{\max}b$ by Cauchy-Schwarz inequality. 
Assumption~\ref{Assump:stochastic}\ref{Assump:anti} is an anti-concentration condition that plays a critical role in controlling the estimation accuracy for the regression coefficient $\beta_\ast$.
In particular, it ensures directions of random feature vectors $X_{a,t}$ from each arm to be uniformly scattered so that the sample covariance matrix $\hat{\Sigma}_t$ will not concentrated in a single direction such as the direction of $\beta_\ast$.
This anti-concentration assumption is mild. For example, it is implied by the boundedness of the probability density function of $\langle X_{a,t},v\rangle$ for each $v\in\mathbb R^d$. Moreover, the parameter $\zeta$ only appears in the regret upper bound via a burn-in term and does not affect the leading term. More details can be found in the discussion after Theorem \ref{Thm:2}.
% Moreover, Theorem \ref{Thm:2} assures that the parameter $\zeta$ only appears in the regret from the ``burn-in" period ($I_2+I_3$) and does not affect the regret terms, $I_1^\alpha$ and $I_4$, that determine the asymptotic growth rate of the regret.

Assumption~\ref{Assump:stochastic}\ref{Assump:eigenvalue} is made to ensure that the design matrix corresponding to each arm is well-behaved so that the sample covariance matrix $\hat\Sigma_t$ satisfies the restricted eigenvalue condition (Definition \ref{Def:restricted_eigenvalue}) with a high probability (c.f.~Proposition~\ref{Prop:2} in Section \ref{Sec:sketched_proof}). The restricted eigenvalue condition is required for the $\ell_1$-error bound analysis of LASSO estimator (Proposition \ref{Prop:1}) and can be implied by the sparse eigenvalue condition (Definition \ref{Def:sparse_eigenvalue}) through the transfer principle (Lemma \ref{lemma:transfer}). Details of the proof can be found in Section \ref{sec:sketched_proof_upper_bound}. Here, $C^\ast$ is a sufficiently large constant that could depend on $K$ and can scale as $\mathcal{O}(K\log K)$ in the worst case. Similar assumptions are commonly adopted in existing work in the high dimensional bandit setting, e.g. \citet{bastani15} and more generally in the high-dimensional statistics literature, e.g. \citet{candes2005} and \citet{wainwright}. This assumption is critical for controlling the estimation accuracy for the regression coefficient $\beta_\ast$ and in turn the regret analysis. 

When samples are independent, several existing results in the literature (e.g. \cite{candes2005} and \cite{wainwright}) show a similar condition, \textit{Restricted Isometry Property} (RIP), on the sample covariance matrix $\hat\Sigma_t$. However, these results on RIP and the corresponding proofs are not applicable in our case as the independence between samples is violated in the bandit setting. In a nutshell, our proof shows that each vector in the restricted cone $\mathbb C(I)=\{v\in\mathbb R^d:\,\|v_{I^c}\|_1\leq 3\|v_I\|_1\}$ with $I=S$ can be well approximated by a $C^\ast s_0$-sparse vector in $\mathbb R^d$ (c.f.~Lemma~\ref{lemma:5} in the appendix), from which we can construct a covering set of this restricted cone with $C^\ast s_0$-sparse vectors of controlled cardinality. Then, we apply concentration inequalities after properly decoupling the dependence structure induced by the bandit policy (c.f.~the proof of Proposition~\ref{Prop:2} in the appendix) to obtain a lower bound on the quadratic form $v^T \hat \Sigma_t v$ for each sparse vector $v\in\mathbb R^d$ in the covering set. Finally, we apply a union bound to prove a lower bound on $v^T \hat \Sigma_t v$ uniformly over all $v\in\mathbb R^d$ belonging to the restricted cone $\mathbb C(S)$, which implies the restricted eigenvalue condition for $\hat \Sigma_t$. Our actual proof is even more delicate than the proof outline described here and guarantees that the restricted eigenvalue does not depend on $K$.  The description after Proposition~\ref{Prop:2} illustrates the high level idea of the actual proof.

To conclude this subsection, we provide a simple example where all our assumptions are satisfied with $(\zeta,\Delta_\ast)$ not depending on the dimension $d$. For each arm, suppose the feature vector follows a truncated multivariate normal distribution truncated on the set $\|X_{a,t}\|_\infty\le 1$ with different mean vectors $\mu_a$ and the same identity matrix as the covariance matrix. Under this setup, the inner product $\langle X_{a,t},v\rangle$ follows a truncated normal distribution with bounded density function, so the first part of Assumption \ref{Assump:stochastic} is satisfied. Assumption \ref{Assump:reward}\ref{Assump:margin} with $\alpha\in[0,1]$ can be verified similarly as the order statistics of $\{\langle X_{a,t},\beta_\ast\rangle : a\in[K]\}$ have bounded density functions. For truncated normal distributions, the optimal region $\Gamma_t$ is well-spread, so the eigenvalue condition (Assumption \ref{Assump:stochastic}\ref{Assump:anti} and \ref{Assump:stochastic}\ref{Assump:eigenvalue}) is also satisfied. In the same example, non-identity covariance matrix can also be allowed if the eigenvalues of the covariance matrix are lower and upper bounded by some absolute positive constants.

% We will provide more discussion on all the assumptions in Section \ref{Sec:assump_discussion} of the appendix along with justification and specific examples.

%{\em Remark.} We will provide a simple example where all our assumptions are satisfied with $(\zeta,\Delta_\ast)$ not depending on the dimension $d$. For each arm, suppose the feature vector follows a truncated multivariate normal distribution truncated on the set $\|X_{a,t}\|_\infty\le 1$ with different mean vectors $\mu_a$ and the same identity matrix as the covariance matrix. Then the inner product $\langle X_{a,t},v\rangle$  still follows a truncated normal distribution with bounded density function, so Assumption \ref{Assump:anti} is satisfied. Assumption \ref{Assump:margin} can be verified similarly as the order statistics of $\{\langle X_{a,t},\beta_\ast\rangle : a\in[K]\}$ have bounded density functions. For truncated normal distribution, the optimal region $\Gamma_t$ is well spread, so the eigenvalue condition (Assumption \ref{Assump:eigenvalue}) is also satisfied. In the same example, non-identity covariance matrix can also be allowed with certain conditions on the eigenvalues of the covariance matrix.

%\section{Lower Bound}
\subsection{Regret Lower Bound}\label{Sec:lower_bound}

In this section we provide a regret lower bound for any policy $\pi$ in the linear bandit environment under the assumptions made in Section \ref{Sec:Assum}. We also provide a sketched proof for Theorem \ref{Thm:3} later in Section \ref{sec:sketched_proof_lower_bound}.

% A natural question that arises is whether our policy is (nearly) optimal in the linear bandit environment under the assumptions made in Section \ref{Sec:Assum}. The following theorem provides an affirmative answer to this question.

\begin{theorem}\label{Thm:3}
Suppose the margin parameter $\alpha \in [0,1]$. Let $\mathcal{P}_{\alpha}$ be the class of distributions of $(X_{a,t}, Y_{a,t}:a\in[d])$, where the feature vector $X_{a,t}$ is drawn from $\mathbb{P}_{a}$ and the reward $Y_{a,t}=\langle X_{a,t}, \beta_{\ast}\rangle + \epsilon_{a,t}$ for $a\in[K]$ satisfies Assumptions 1--2 with parameter $\alpha$ in Assumption~\ref{Assump:stochastic}\ref{Assump:margin}. Then for large enough horizon $T$,
\begin{equation}
  \inf_{\pi} \sup_{P_{X,Y} \in \mathcal{P}_{\alpha}}  R_{T}\left(\pi, \pi^{*}\right) \geq C_L \left((\log d)^{\frac{\alpha+1}{2}}T^{\frac{1-\alpha}{2}}+\log T\right),
\end{equation}
where $C_L$ is a constant independent of $T$ and $d.$
\end{theorem}

Our Theorem \ref{Thm:3} provides a more general result compared to existing results in the literature considering both the dimension $d$ and horizon $T$ by leveraging the margin condition (Assumption \ref{Assump:reward}\ref{Assump:margin}). The theorem shows a poly-logarithmic dependence on dimension $d$. In addition, it unifies the polynomial and logarithmic dependence on horizon $T$ through the margin parameter $\alpha$, which describes the hardness of distinguishing the optimal arm from the sub-optimal arms. In particular, when $\alpha=1$, the lower bound scales as $\Omega(\log d + \log T)$, which is logarithmic in both $d$ and $T$; when $\alpha=0$, the lower bound scales as $\Omega\left(\sqrt{T\log d}\right)$, which is poly-logarithmic in $d$ and polynomial in $T$.

For comparison, \citet{golden09} prove a lower bound of $\Omega(T^{\frac{1-\alpha}{2}})$ for one-dimensional one-armed linear bandit problem, where they assume a margin condition corresponding to Assumption \ref{Assump:reward}\ref{Assump:margin} with $\alpha\in (0, 1]$. In addition, they propose a policy with forced sampling, which achieves optimal regret in horizon $T$, i.e., $\mathcal{O}(T^{\frac{1-\alpha}{2}})$, when $\alpha\in(0, 1)$. Later, \citet{golden} prove a lower bound of $\Omega(\log T)$ for low-dimensional linear bandit problem, where they also assume a margin condition which corresponds to the $\alpha=1$ case in Assumption \ref{Assump:reward}\ref{Assump:margin}. Meanwhile, they propose the OLS-bandit algorithm with $\mathcal{O}(d^3\log T)$ regret upper bound, which is sub-optimal in $d$ according to Theorem \ref{Thm:3}. \citet{paat2010} also prove a regret lower bound of $\Omega(d\sqrt{T})$ for low-dimensional linear bandit problem, where the covariate set $\mathcal{D}_t$ is a compact set of infinitely many feature vectors, e.g., the unit sphere $S^{d-1} \subseteq \mathbb{R}^d$. Later, \citet{chu} prove a lower bound of $\Omega(\sqrt{Td})$ for low-dimensional contextual bandit problem, where no margin condition is assumed and the regret lower bound therein corresponds to the worst case when $\alpha=0$, and propose the LinUCB algorithm with $\tilde{\mathcal{O}}(\sqrt{dT})$ regret upper bound, which is near-optimal according to Theorem \ref{Thm:3}. Our lower bound result unifies the existing results through the parameter $\alpha$ and applies to the high-dimensional linear bandit problems. More details of the comparison are deferred to Section \ref{Sec:literature}.

\begin{remark}
The problem formulation considered in \citet{golden} and \citet{bastani15} is slightly different from ours, where they assume that at each round, there is one common feature vector shared by all arms and each arm has its own parameter vector. Their formulation can be mathematically reparametrized into our formulation. In particular, they assume that each of the $K$ arms has its own parameter vector $\beta_{\ast,a}\in \mathbb{R}^d$ for $a\in[K]$. The common feature vector $X_t$ in each round is drawn i.i.d. from a distribution. Then, this formulation can be embedded into our formulation. Specifically, the feature vector $X_{a,t}$ for the $a$-th arm is a $Kd$-dimensional vector defined as $X_{a,t}:=(0,\cdots, 0, X_t^T, 0, \cdots, 0)^T$, where the $a$-th block of $X_{a,t}$ is $X_{t}$ and zero elsewhere. The common parameter $\beta_{\ast}:=(\beta_{\ast,1}^T, \cdots, \beta_{\ast,K}^T)^T$ in our setting is a concatenation of all vectors $\beta_{\ast,a}$ for $a\in[K]$. Due to the reparametrization, the regret lower bound result, i.e., $\Omega(\log T)$, proved in \citet{golden} can also applied in our formulation.
\end{remark}
% Meanwhile, they proposed an OLS-bandit algorithm with regret upper bound $\mathcal{O}(d^3\log T)$. More recently, \citet{bastani15} proposed a LASSO-bandit algorithm with regret $\mathcal{O}(\log^2 T)$ for high-dimensional linear bandit problem under $\alpha=1$ case, which has a $\log T$ discrepancy with the regret lower bound.

%{\color{blue} For comparison, \citet{golden} proved a lower bound of $\mathcal{O}(\log T)$ for low-dimensional linear bandit problem under $\alpha=1$ case. Here we prove a more general result with dependence both on dimension $d$ and horizon $T$, by leveraging the margin condition (Assumption \ref{Assump:margin}). Moreover, from the result we can see that the exponent $\alpha$ in margin condition has an impact on the lower bound of expected regret. When $\alpha=1$, the lower bound scales $\mathcal{O}(\log d + \log T)$ in dimension $d$ and horizon $T$. When $\alpha$ is closer to $0$, the bandit becomes more difficult according to margin condition, which leads to polynomial dependence on T in the lower bound. In the worst case when $\alpha=0$, the lower regret bound will scale as $\mathcal{O}(\sqrt{T\log d})$. This coincides the dependence on $T$ with \citet{chu} in low-dimensional contextual bandit problem where no margin condition is assumed.}

%{\color{red}In the lower bound, how we can unify the results. Consider the worst case, where we have the polynomial dependence on $d$ or $T$, and relation with the classification literature.}

\section{$\ell_1$-Confidence Ball Algorithm for High-dimensional Linear Bandit}
\label{Sec:algorithm}

In this section we propose a conceptually simple and computationally efficient bandit algorithm. The proposed algorithm is nearly optimal---it matches the regret lower bound when margin parameter $\alpha$ belongs to $[0, 1)$ and is optimal up to a factor of $\log T$ when $\alpha=1$.

\subsection{Optimism in the Face of Uncertainty}
\label{Sec:Confidence_set}

In this section, we motivate our $\ell_1$-confidence ball based method under the general framework of optimism in the face of uncertainty (OFU) introduced in \citet{auer} and \citet{abbasi11}. More precisely, suppose that at the beginning of each round, say $t+1$, we can construct a confidence set $\mathcal C_{t+1}\subseteq \mathbb{R}^d$ for the parameter $\beta_\ast$ using the past selected features $X_{\hat{a}_1,1},\ldots, X_{\hat{a}_{t},t}$ and observed rewards $Y_1, \ldots, Y_{t}$ up to round $t$, so that $\beta_\ast$ belongs to $\mathcal C_{t+1}$ with high probability.
Then a OFU-based algorithm chooses the action-estimate pair 
\begin{equation}\label{eq:2}
    (X_{\hat{a}_{t+1},t+1},\tilde{\beta}_{t+1})=\underset{(x,\beta)\in \mathcal{D}_{t+1}\times \mathcal{C}_{t+1}}{\arg\max} \langle x,\beta\rangle,
\end{equation}
which jointly maximizes the expected reward within the covariate set $\mathcal{D}_{t+1}=\{X_{a,t+1}\}_{a=1}^K$ and the confidence set $\mathcal{C}_{t+1}$ of $\beta_{\ast}$. The shape and size of the confidence set $\mathcal C_{t+1}$ reflects our uncertainty on the unknown parameter $\beta_\ast$.
In high-dimensional linear bandits, we consider to use the LASSO \citep{lasso} estimator as the center of $\mathcal C_{t+1}$, since it adapts the sparsity structure of the parameter. As a remark, other penalized estimators such as MCP \citep{zhang2010nearly} and SCAD \citep{fan2001variable} can also be considered as the center for the confidence set, though need more sophisticated regret analysis.

At the beginning of round $t+1$, we form the LASSO estimator $\hat\beta_t$ for $\beta_\ast$ by viewing the observations as coming from the linear model $Y=X\beta_\ast+\epsilon$, with design matrix $X=X_t\in \mathbb{R}^{t\times d}$, response vector $Y=Y_{t}\in\mathbb{R}^{t}$, both defined in Section~\ref{Sec:formulation}, and noise vector $\epsilon\in \mathbb{R}^{t}$ whose entries are independent $\sigma$-sub-Gaussian random variables. Given a regularization parameter $\lambda\ge 0$, the LASSO estimator is the solution of
\begin{equation}\label{eq:4}
        \hat{\beta}_t\in \arg\min_{\beta}\{\mathcal{L}_t(\beta)+\lambda\|\beta\|_1\},
\end{equation}
where $\mathcal{L}_t$ denotes the empirical loss function at round $t$. In our regression framework, the loss function is the least squares objective $\mathcal{L}_t(\beta)=\frac{1}{2t}\|Y-X\beta\|_2^2$. Proposition~\ref{Prop:1} in Section \ref{sec:sketched_proof_upper_bound} suggests a choice  of the regularization parameter as $\lambda=\lambda_t=2\sigma x_{\max}\sqrt{(2\log t+2\log d)/t}$, where constant $x_{\max}$ (as well as the $\phi_0$ in the following display) is defined in Section~\ref{Sec:Assum}.

% The LASSO estimator $\hat{\beta}_t$ will be the center of the confidence set $\mathcal C_{t+1}$. As argued in the previous subsection, we would like to construct $\mathcal C_{t+1}$ as an $\ell_1$-ball centered at $\hat{\beta}_t$.

Having determined the center of the confidence set $\mathcal C_t$, it remains to determine its shape. Our proposal of the $\ell_1$ ball as the shape of the confidence set is motivated by both theoretical and computational considerations: (i) Theoretically, the shape of the confidence set is important because the cumulative regret can be roughly characterized by the sum of the ``volumes" of the confidence sets $\mathcal C_t$ over the time horizon. In the low-dimensional regime where $d\ll T$,  since all norms in the Euclidean space are equivalent (up to some constant depending on the dimension), the shape of the ball induced by a norm has limited impact on the regret. In contrast, as the dimension $d$ grows, unit balls under different norms have drastically different volumes making the choice of the shape crucial. For example, the volume ratio between unit $\ell_1$-ball and unit $\ell_2$-ball is $O(d^{-d/2})$, indicating the importance of using the right ``norm" for constructing $\mathcal C_t$. As we will show in Corollary \ref{Cor:confidence_set} in the appendix, the use of $\ell_1$-ball improves the dependence of the regret on the feature dimension from $d$ to $s_0$ modulo $\log d$ terms where recall that $s_0$ is the sparsity level of the model. (ii) Computationally, we want to maintain the convexity of the optimization problem~\eqref{eq:2} by requiring the confidence set $\mathcal C_t$ to be convex. These two factors together motivate us to use a properly centered $\ell_1$-ball for $\mathcal C_t$, since $q=1$ is the smallest number maintains the convexity of the $\ell_q$-ball over $q\geq 0$. 

Corollary~\ref{Cor:confidence_set} in the appendix provides a high probability upper bound to the $\ell_1$-distance $\|\hat\beta_t-\beta_\ast\|_1$ between $\hat\beta_t$ and $\beta_\ast$, motivating us to use the following confidence set 
\begin{equation}\label{eq:5}
    \mathcal{C}_{t+1}=\left\{\beta\in \mathbb{R}^d: \|\hat{\beta}_t-\beta\|_1\le \tau_{t} \right\},
\end{equation}
where $\tau_{t}  = \frac{384s_0 \sigma x_{\max}}{\phi_0^2}\sqrt{\frac{2\log t+2\log d}{t}}.$
%{\color{red} Make sure that the definition does not go out of margin. May be split the equation into two lines. }

Note that the size of the confidence set scales poly-logarithmically in feature dimension $d$ (i.e. $\mathcal{O}(\sqrt{\log d})$), which improves upon the $\mathcal{O}(\sqrt{d})$ size of the confidence ellipsoid in \citet{abbasi11} that has a polynomial dependence on dimension $d$. This is because we have leveraged on the sparsity structure of $\beta_\ast$ captured by the LASSO estimator.

\subsection{$\ell_1$-Confidence Ball Based Algorithm}

In this subsection, we formally describe our algorithm for the high-dimensional linear bandit problem, which is summarized in Algorithm~\ref{alg:l_1_ball} below.
The algorithm takes as inputs an initial regularization parameter $\lambda_0$ and an initial diameter of the confidence set $\tau_0$. Recall from the previous subsection that the $\ell_1$-confidence ball at time $t$ is $\mathcal{C}_t=\{\beta: \|\beta-\hat{\beta}_{t-1}\|_1\le \tau_{t-1}\}$, where $\hat{\beta}_{t-1}$ is the LASSO estimator given by~\eqref{eq:4}, and $\tau_{t-1}$ is given in~\eqref{eq:5}. The action selection rule~\eqref{eq:2} of the algorithm can be reformulated as
\begin{align}\label{eq:3}
    X_{\hat{a}_t,t}& =\arg\max_{x\in \mathcal{D}_t} \big\{\max_{\beta\in \mathcal{C}_t}\langle x, \beta\rangle\big\} 
    = \arg\max_{x\in \mathcal{D}_t} \big\{\max_{\|u\|_1\le \tau_{t-1}}\langle x, \hat{\beta}_{t-1}+u\rangle\big\} \nonumber \\
    &=\arg\max_{x\in \mathcal{D}_t} \big\{\langle x, \hat{\beta}_{t-1}\rangle + \tau_{t-1}\|x\|_{\infty}\big\}.\\[-5ex] \nonumber
\end{align}

The criterion function in the second line in equation~\eqref{eq:3} is composed of a point estimate of the expected reward plus the confidence width of $x$ through its $\ell_\infty$ norm, making the optimization simple and the algorithm computationally efficient. The confidence width, also known as the exploration bonus, encourages the learner to visit new areas in the feature space to attain an optimal exploration versus exploitation trade-off. Computation-wise, the OLS-bandit algorithm of \citet{golden} and the LASSO-bandit algorithm of \citet{bastani15} require a pre-processing step to select a subset of arms based on forced-sample estimation; the OFUL-LS algorithm of \citet{abbasi11} needs optimization over an ellipsoid that involves matrix inversion, which makes it computationally expensive.

\begin{remark} 
Our proposed algorithm reduces to a greedy algorithm when all $X_{a,t}$'s have the same $\ell_\infty$-norm. However,
in cases when $X_{a,t}$'s have very different magnitudes across arms, their sup-norms are able to reflect the prediction uncertainty in the point estimates $\langle X_{a,t}, \hat{\beta}_{t-1}\rangle$'s. As a consequence, our algorithm encourages explorations on those arms with high uncertainty while the greedy algorithm completely ignores this information and tends to exhibit poor performance.
\end{remark}
% In practice, we can standardize the columns of the design matrix before solving equation \eqref{eq:4} at each time step, which will account for the number of times an arm is pulled and the uncertainty in each feature direction. Then the standardized $\|X_{a,t}\|_{\infty}$ will be different for different arms.

% {\color{red} \em{Remark}: According to equation \eqref{eq:3}, our proposed algorithm will reduce to the greedy method under the scenario where all arms share the same feature vector $X_t$ and different arms have different parameter vectors $\beta^{a}_{\ast}$. This is because when we embed this special case into the general setup of this paper, we have $X_{a,t}=(0,\ldots,X_t^T,\ldots, 0)^T$ and $\beta_{\ast} = ((\beta^{1}_{\ast})^T, \ldots, (\beta^{K}_{\ast})^T)^T$. Then the $\|X_{a,t}\|_{\infty}$ will be the same for all arms. In practice, we can standardize the columns of the design matrix before solving equation \eqref{eq:4} at each time step, which will account for the number of times an arm is pulled and the uncertainty in each feature direction. Then the standardized $\|X_{a,t}\|_{\infty}$ will be different for different arms. }

\begin{algorithm}
   \caption{$\ell_1$-Confidence Ball Based Algorithm}
   \label{alg:l_1_ball}
\begin{algorithmic}[1]
   \Require Initial regularization parameter $\lambda_0$, and initial diameter $\tau_0$ for the confidence set.
   %\REPEAT
   \State Initialize $\hat{\beta}_0$ by $0$ in $\mathbb{R}^d$.
   \For{$t=1$ {\bfseries to} $T$}
   %\IF{$x_i > x_{i+1}$}
   \State Observe $\mathcal{D}_t=\{X_{a,t}\}_{a=1}^K$.
   \State Obtain $\hat{\beta}_{t-1}$ by solving equation~\eqref{eq:4}, i.e.  $\hat{\beta}_{t-1}\in \arg\min_{\beta}\{\mathcal{L}_{t-1}(\beta)+\lambda_{t-1}\|\beta\|_1\}$.
   \State Construct confidence set $\mathcal{C}_t=\{\beta\in \mathbb{R}^d:\|\hat{\beta}_{t-1}-\beta\|_1\le \tau_{t-1}\}$.
   \State Select feature $X_{\hat{a}_t,t}$ from equation~\eqref{eq:3}, and observe $Y_t=\langle X_{\hat{a}_t,t}, \beta_\ast\rangle +\epsilon_t$.
   \State Set $\lambda_t=\lambda_0\sqrt{\frac{\log d+\log t}{t}}$ and $\tau_t=\tau_0\sqrt{\frac{\log d+\log t}{t}}$.
   %\ENDIF
   \EndFor
   %\UNTIL{$noChange$ is $true$}
\end{algorithmic}
\end{algorithm}

\subsection{Regret Analysis of Proposed Algorithm}\label{Sec:theory}

%\subsection{$\ell_1$-risk Bound for LASSO Estimator}\label{Sec:Risk}
% \subsection{Regret Bound}\label{Sec:Risk}

In this subsection, we provide our theoretical regret analysis. The following theorem provides an upper bound on the cumulative regret for the proposed $\ell_1$-confidence ball based algorithm. 

%We will provide more details of the proof in Section \ref{Sec:thm1_proof} of the appendix.}

%which stats that with probability at least $1-\delta$, for all $t\ge 0$,
%\begin{equation*}
%    \|\hat{\beta}_t-\beta_\ast\|_{\bar{V}_t}\le \sigma\sqrt{d\log\left(\frac{1+tL^2/\lambda}{\delta}\right)}+\lambda^{1/2}S,
%\end{equation*}
%where $\|X_{a,t}\|_2\le L$, $\|\beta_\ast\|_2\le S$. 

%\subsection{Regret Bound}
%It also guarantees that the number of times when optimal arms are selected is large enough. In particular, the optimal arm will be selected at least a constant fraction of times.

%\textcolor{red}{examples are provided at the end here. Also, this example is way to simple, need more} 
%\textbf{Simple example:} Let the probability distribution of features $\mathbb{P}_{a}$ for each arm be the uniform distribution over the $d$-dimensional unit cube $[0,1]^d$. Consider the true parameter vector as $\beta_\ast=(1,0,\ldots,0)$. 
%
%Our example certainly satisfies the assumption with $x_{\max}=1$ and $b=1$.
%In our example, since the uniform distribution has bounded density everywhere, the density for $\langle X_{a,t}, v\rangle^2$ is also bounded. So we can find a constant $\zeta$ that satisfies the assumption. In our example, since the uniform distribution has bounded density everywhere, this assumption is satisfied. Specifically, the choice of $C_1=1$ is sufficient.

\begin{theorem} \label{Thm:2}
Suppose that Assumptions 1--2 hold and $K\ge 2$, $d\ge 1$, $\lambda_0=2\sqrt{2}\sigma x_{\max}$ and $\tau_0=\frac{384\sqrt{2}s_0 \sigma x_{\max}}{\phi_0^2}$. The expected cumulative regret of our algorithm at time horizon $T$ can be bounded as
\begin{align}\label{eq:6}
    R_T \leq & I_1^{\alpha} + \overbrace{C_3 x_{\max}b{\Lambda_0^2} (K\log K) s_0\log d}^{I_2}  \nonumber\\
     &~~~ + \overbrace{C_4 \sigma^2x_{\max}^5 b s_0^2\zeta^2K^2\log d/{\Delta_\ast^2}}^{I_3} + \overbrace{C_5x_{\max}b\log T}^{I_4},
\end{align}
where
\begin{align*}
    I_1^{\alpha} = \left\{\begin{array}{ll}
        C_2\frac{s_0^{\alpha+1}\sigma^{\alpha+1}x_{\max}^{2(\alpha+1)}}{\Delta_\ast^{\alpha}\phi_0^{2(\alpha+1)}}(\log d)^{\frac{\alpha+1}{2}} T^{\frac{1-\alpha}{2}}, & \text{when}~~\alpha\in[0,1), \\
        C_2\frac{s_0^2\sigma^2x_{\max}^4}{\Delta_\ast\phi_0^4}[\log d+\log T]\log T, &  \text{when}~~\alpha=1,\\
        C_2\frac{s_0^{2}\sigma^{2}x_{\max}^{4}}{(\alpha-1)\Delta_\ast \phi_0^{4}}\log d, & \text{when}~~\alpha\in(1, +\infty), \\
        C_2\frac{s_0^{2}\sigma^{2}x_{\max}^{4}}{\Delta_\ast \phi_0^{4}}\log d, & \text{when}~~\alpha=+\infty,
    \end{array}\right.
\end{align*}
% where $I_1^{\alpha}=C_2\frac{s_0^2\sigma^2x_{\max}^4}{\Delta_\ast\phi_0^4}[\log d+\log T]\log T$ for $\alpha=1$; $I_1^\alpha=C_2\frac{s_0^{\alpha+1}\sigma^{\alpha+1}x_{\max}^{2(\alpha+1)}}{\Delta_\ast^{\alpha}\phi_0^{2(\alpha+1)}}(\log d)^{\frac{\alpha+1}{2}} T^{\frac{1-\alpha}{2}}$ for $\alpha\in[0,1)$, 
which is the leading term of the regret bound that depends on the margin parameter $\alpha$ defined in Assumption~\ref{Assump:reward}\ref{Assump:margin}, and 
$C_2, C_3, C_4$, $C_5$ are positive constants not depending on $T$, $d$ or $\alpha$. %{\color{red} Constant $C_2$ depends on the constant $\Delta_\ast$ of Assumption~\ref{Assump:margin}.}
\end{theorem}

The regret bound provided by Theorem \ref{Thm:2} for our algorithm shows that the regret of the algorithm grows poly-logarithmically in $d$, i.e., $R_T=\mathcal{O}((\log d)^{\frac{\alpha+1}{2}})$, when $\alpha\in[0,1)$; logarithmically in $d$, i.e., $\mathcal{O}(\log d)$, when $\alpha\in[1, +\infty]$. Meanwhile, the cumulative regret grows polynomially in $T$, i.e., $R_T=\mathcal{O}(T^{\frac{1-\alpha}{2}})$, when $\alpha \in [0, 1)$; poly-logarithmically in $T$, i.e., $R_T=\mathcal{O}((\log T)^2)$, when $\alpha=1$; and logarithmically in $T$, i.e., $R_T=\mathcal{O}(\log T)$, when $\alpha\in (1, +\infty]$. In addition, for $\alpha\in (1,+\infty)$, the term $I_{1}^{\alpha}$ will increase as $\alpha\rightarrow 1_{+}$ and scale as $\mathcal{O}(\log d \log T)$ when $\alpha=1$ in the extreme case, which fills the gap between $\alpha=1$ and $\alpha\in(1,+\infty)$ cases.
Comparing the regret upper bound for our proposed method with the minimax regret lower bound established in Section \ref{Sec:lower_bound}, we can see that our method is optimal up to a factor of $\log T$ when $\alpha \in [0,1]$. %In the worst case when $\alpha=0$, the lower regret bound will scale as $\mathcal{O}(\sqrt{T\log d})$. 
%This bound compares with the $\mathcal{O}([\log d+\log T]^2)$ by Bastini \& Bayati (2015).

Intuitively, 
%the second and third terms 
$I_2+I_3$ on the right-hand side of display~\eqref{eq:6} describes the regret caused by the ``burn-in'' period of exploring the space of the feature space which does not contribute to the asymptotic regret growth, while the last term $I_4$ is the cumulative regret when $\mathcal{C}_{t+1}$ does not contain $\beta_\ast$ in Corollary \ref{Cor:confidence_set}. In the regret bound, constant $\Delta_\ast$ (defined in Assumption~\ref{Assump:reward}\ref{Assump:margin}) plays the role of the gap parameter as appeared in a typical problem-dependent regret bound \citep{abbasi11}. The first term is due to the risk of selecting sub-optimal arms when the features fall near the decision boundary, which is controlled by the margin condition (Assumption~\ref{Assump:reward}\ref{Assump:margin}). From the result of Theorem \ref{Thm:2}, we can see that when $\alpha$ is larger, the feature vectors are less likely to fall into the close neighborhood of the decision boundary and the bandit environment will become easier to learn. In particular, for the $\alpha=+\infty$ case when there is a deterministic positive gap $\Delta_{\ast}$ between the rewards of the optimal arm and sub-optimal arms, the regret bound for the algorithm scales polynomially in both dimension $d$ and horizon $T$, i.e., $\mathcal{O}(\log d + \log T)$. We discuss how our regret bound compares with and provides a unified view of the existing results in Section \ref{Sec:literature}.

\subsection{Comparisons with Existing Literature}\label{Sec:literature}

\subsubsection{Problem Formulation}
The stochastic linear bandit problem was first introduced by \citet{auer}, and was subsequently studied by \citet{dani}, \citet{paat2010} and \citet{chu}. Later, \citet{abbasi11} and \citet{abbasi12} proposed the OFUL algorithm for both low-dimensional and high-dimensional settings. The model parametrization considered in \citet{abbasi11}, \citet{dani} and \citet{paat2010} is the same as ours. However, there is one major difference between the two formulations that the covariate set in their problem consists of infinitely many feature vectors, e.g., the unit sphere $S^{d-1}\subseteq\mathbb{R}^d$, while the covariate set $\mathcal{D}_t$ in our paper consists of finitely many feature vectors drawn from some underlying distribution. Moreover, the covariate set in \citet{dani} and \citet{paat2010} does not change over time, therefore the optimal arm will be the same in different rounds. In contrast, the optimal arm in our setting varies along the time due to the randomness in the feature vectors. In addition, the feature space can be explored implicitly in our setting and no explicit exploration phase is needed as in \citet{paat2010}.

In another parallel strand of linear bandit research, \citet{golden} proposed an OLS-bandit algorithm for low-dimensional multi-armed bandit problems, where the formulation is different from ours. In particular, they consider the $K$-armed bandit problem with bounded linear response, where feature vector at each time is common to all arms and parameters are arm-specific. After that, \citet{bastani15} proposed a Lasso-bandit algorithm for the high-dimensional linear bandit problem which requires forced sampling and prior knowledge on their gap parameter. Later, in the same setting as \citet{bastani15}, \citet{wang18} introduced an MCP-bandit algorithm which also requires forced sampling and knowledge of gap parameter. 
As discussed in the remarks after Theorem 1, our problem formulation is also different from that in \citet{bastani15}. The formulation in our paper is meaningful on its own in real world applications. In some real applications, many baseline features are available which have the same effect in the sense of sharing the same regression coefficient $\beta_{\ast}$ across different arms. In such cases, our formulation can help estimate the parameters associated with the common features more accurately by using all the data. However, the formulation in \citet{bastani15} does not incorporate such structure, which will cause loss of information. For example, in clinical trials or advertisement problems, we can map the interactions between the features of patients (or customers) and treatments (or advertisements) as a feature vectors $X_{a,t}$ for each arm, with one common parameter vector $\beta_{\ast}$ across arms. In these applications, the parameters corresponding to the main effect of patients (or customers) such as age and gender are the same for all arms. Our formulation facilitates sample-efficient estimation of these parameters by combining data from multiple arms. In comparison, the parameters under the formulation in \citet{bastani15} are estimated separately for different arms. In addition, the difference between the two formulations has been considered in \citet{kim2019}, where they observe that our formulation is advantageous in such practical applications. In particular, when the number of arms is large, it is not practical to apply the formulation in \citet{bastani15} due to the large number of parameters. When the set of arms changes over time such as in online advertisement, it is also not feasible to assign a different parameter for every new incoming item as in \citet{bastani15}.

Apart from the above benefits, the motivation for our setting comes from applications such as online advertisements, where there are finitely many available products (arms) for each incoming customer. In addition, available products are usually different for different customers and drawn from some underlying distribution, where each arm is a category of products that is associated with a distribution over the set of products.

\subsubsection{Regret Bound Analysis}
In terms of the regret bound analysis, \citet{paat2010} proved a lower bound of $\Omega(d\sqrt{T})$ for both cumulative regret and Bayes risk in low-dimensional linear bandit problem, where the covariate set $\mathcal{D}_t$ is a compact set of infinitely many feature vectors, e.g., the unit sphere $S^{d-1}\subseteq \mathbb{R}^d$. They also proposed an algorithm with matching regret upper bound that alternates between exploration and exploitation phases. \citet{chu} proposed the LinUCB algorithm, which has the regret upper bound $\mathcal{O}(\sqrt{Td\log(T \log(T)/\delta})$ with probability $1-\delta$, and proved a regret lower bound of $\Omega(\sqrt{dT})$ for the low-dimensional contextual bandit problem with linear payoff, which corresponds to our $\alpha=0$ case. \citet{abbasi11} proved that the expected cumulative regret of the OFUL algorithm scales as $\mathcal{O}(d\sqrt{T})$ in both low-dimensional and high-dimensional settings. In the above literature, the feature vectors are allowed to be chosen arbitrarily by an adversary and the analysis is based on the worst case, leading to the polynomial dependence on dimension $d$. In contrast, \citet{golden} proved a cumulative regret of $\mathcal{O}(d^3\log T)$ for their proposed OLS-bandit algorithm and a regret lower bound of $\Omega(\log T)$ corresponding to the margin condition (Assumption~\ref{Assump:reward}\ref{Assump:margin}) with $\alpha=1$. However, the cumulative regret of the above algorithms scales polynomially in dimension $d$, which can be sub-optimal in the high-dimensional setting. The Lasso-bandit algorithm in \citet{bastani15} and the MCP-bandit algorithm in \citet{wang18} are shown to have an improved regret upper bound $\mathcal{O}(\log d)$ in dimension $d$ under the margin condition with $\alpha=1$, but require forced sampling, which could be costly in some practical cases such as medical and marketing applications. In comparison, our proposed algorithm (Algorithm \ref{alg:l_1_ball}) does not require any forced sampling or knowledge of the gap parameter. Meanwhile, our theoretical results on regret upper bound can be applied to any margin parameter $\alpha\in[0,+\infty]$. On the other hand, the setting of \citet{wang18} and  \citet{bastani15} corresponds only to the case of $\alpha=1$, for which our algorithm achieves the same regret bound dependence on $d$, i.e., $\mathcal{O}(\log d)$.

\subsection{Proof Sketch for the Main Results}\label{Sec:sketched_proof}

In this section, we outline the proofs for Theorem \ref{Thm:3} and Theorem \ref{Thm:2}.

\subsubsection{Proof Sketch for Theorem \ref{Thm:3}}
\label{sec:sketched_proof_lower_bound}
In this section, we provide a sketch of the proof for Theorem \ref{Thm:3}. In \citet{golden}, the author already proved a lower bound of $\Omega(\log T)$ under $\alpha=1$ case. In this paper, we prove a lower bound on dimension $d$ and also the polynomial dependence on horizon $T$ when margin condition parameter $\alpha<1$. In order to prove the lower bound result, we follow the standard proof technique by designing a class $\mathcal{P}_{\mathcal{M}}$ of distributions, where $\mathcal{M}$ denotes the parameter space for the regression coefficient vector $\beta$ in the linear model $Y_{a,t}=\langle X_{a,t}, \beta\rangle +\epsilon_{a,t}$, and then reducing the worst case of cumulative regret in $\mathcal{P}_{\mathcal{M}}$ to the average case by introducing a prior distribution over the class $\mathcal{P}_{\mathcal{M}}$. 
The novelty of our proof comes from the design of the class $\mathcal{P}_{\mathcal{M}}$, which is a hard-to-learn subset of $\mathcal{P}_{\alpha}$ that yields the lower bound. Specifically, $\mathcal{P}_{\mathcal{M}}$ is composed of bandit problems with two arms, the first of which is $X^{(1)}=(X_0,X_1, \ldots, X_d)$ and the second is $X^{(2)} = (0, -X_1,\ldots, -X_d)$. Here $X_0$ is a random variable taking values in $\{-1, 0, 1\}$, while $X_i$ for $i\in\{1,\cdots,d\}$ follows truncated normal distribution of $N(0,1)$. We take $\mathcal{M}=\{\beta=(\beta_0, \beta_1, \cdots, \beta_d)\in \mathbb{R}^{d+1}: \beta_0=1,\beta_u=\theta, \beta_j=0 \; \mathrm{for}~j \neq u, u\in[d]\}$ as a set of true parameter $\beta^{\ast}$, where $\theta$ is a tuning parameter controlling the signal strength of the second largest component of $\beta$. Here we assume that $\beta_0=1$ is known and does not need to be estimated, and the other nonzero entry $\beta_u=c\sqrt{\frac{\log d}{T}}$ with a small constant $c$. The choice of $\mathcal{M}$ is to make sure that assumptions in Section \ref{Sec:Assum} are satisfied and the magnitude of $\beta_u$ is small so that it is hard for policy to distinguish it from $0$. Moreover, we specify a uniform prior distribution on the parameter vector $\beta_{\ast}$ over the set $\mathcal{M}$. According to Bayesian decision rule, we prove that the average cumulative regret can be lower bounded by the regret incurred by the Bayesian optimal policy. By following the common proof technique of reducing the problem of obtaining the estimation lower bound to a multiple hypothesis testing problem, we show that learning $\beta$ from space $\mathcal{M}$ reduces to a variable selection problem that causes the poly-logarithmic term in $d$ and the polynomial term in $T$ in our new lower bound. Specifically, the multiple hypothesis testing problem is to determine the index $u$ of the nonzero entry $\beta_u$ from $\{1,\cdots,d\}$ and the optimal Bayes rule is $\hat u_{t-1} =\arg\max_{j\in [d]}|\hat\beta_{j,t-1}|$, where $\hat{\beta}_{t-1}$ is defined according to the Bayesian policy. Finally, by applying Fano's Lemma (Lemma \ref{lem:fano}), we derive the lower bound for the probability of error $P(\hat{u}_{t-1}\neq u)$ in multiple hypothesis testing problem that leads to the regret lower bound.

\subsubsection{Proof Sketch for Theorem \ref{Thm:2}}
\label{sec:sketched_proof_upper_bound}
In this section, we provide the proof strategy for Theorem \ref{Thm:2}. In order to prove Theorem \ref{Thm:2}, we need to prove in Corollary \ref{Cor:confidence_set} that the $\ell_1$-confidence set $\mathcal{C}_{t}$ in \eqref{eq:5} contains $\beta_{\ast}$ with high probability for properly selected regularization parameters. Since we choose LASSO estimator as the center of the confidence set $\mathcal{C}_{t}$, we need to analyze the error bound of the LASSO estimator. This is challenging since the observed data are highly dependent from each other due to bandit policy, thus we cannot directly apply the standard error analysis of the LASSO estimator. In particular, the restricted eigenvalue condition for $\hat{\Sigma}_t$, which is required for the analysis of LASSO estimator, is hard to prove for dependent data. In this paper, we introduce a convergence result (Proposition \ref{Prop:1}) on LASSO estimator for dependent data and a novel method to prove the eigenvalue condition (Proposition \ref{Prop:2}). 
% After that, Corollary \ref{Cor:confidence_set} is a direction application of the two propositions.
Finally, we derive the cumulative regret taking into account of the regret incurred under different circumstances.

Firstly, we provide an $\ell_1$-error bound for the LASSO estimator. Specifically, Proposition~\ref{Prop:1} shows that the LASSO estimator falls into an $\mathcal{O}(\sqrt{\log d / t})$ neighborhood around the true parameter $\beta_{\ast}$ with high probability.
Before the statement of Proposition \ref{Prop:1}, we introduce the definition of \emph{restricted eigenvalue condition} for positive semidefinite matrices.

\begin{definition}[Restricted eigenvalue condition]\label{Def:restricted_eigenvalue}
 For any set of indices $I\subseteq [d]$ and a positive constant $\phi$, define the set 
\begin{align*}
    \mathcal{C}(I,\phi)=\{M\in \mathbb S^{d\times d}\;|\;\forall v\in \mathbb{R}^d~s.t.~\|v_{I^c}\|_1\le 3\|v_I\|_1,~\text{we have}~\|v\|_1^2\le |I|(v^TMv)/\phi^2\}.
\end{align*}
 A matrix that belongs to this set is said to satisfy the restricted eigenvalue condition over index set $I$ with parameter $\phi$.
\end{definition}

On the event that the sample covariance matrix $\hat{\Sigma}_t$ at time $t$ belongs to $\mathcal{C}(S,\phi)$ for some suitable constant $\phi$, Proposition~\ref{Prop:1} extends the usual LASSO analysis for i.i.d.~observations to our bandit setting where observations are dependent due to the sequential decision policy.

\begin{proposition}\label{Prop:1}
Let $X_i$ denote the $i^{th}$ row of $X$ and $Y(i)$ denote the $i^{th}$ entry of $Y$. The sequence $\{X_i,i=1,\cdots,t\}$ forms an adapted sequence of observations, i.e., $X_i$ depends on past features and their rewards $\{X_{i'}. Y(i')\}_{i'=1}^{i-1}$. Also, assume that all $X_i$ satisfy $\|X_i\|_{\infty}\le x_{\max}$ and the regularization parameter $\lambda=2\sigma x_{\max}\sqrt{(2\log t+2\log d)/t}$. Then for any $\phi>0$, we have
\begin{align*}
    & P\left[\|\hat{\beta}-\beta_\ast\|_1\le  \frac{6s_0 \sigma x_{\max}}{\phi^2}\sqrt{\frac{2\log t+2\log d}{t}} \right] \\
    \geq& 1-\frac{2}{t}-P[\hat{\Sigma}(X)\notin \mathcal{C}(supp(\beta_\ast),\phi)],
\end{align*}
where $\hat\beta$ is the LASSO estimator by solving equation~\eqref{eq:4} with $X,Y$ and regularization parameter $\lambda$.
\end{proposition}
Proposition \ref{Prop:1} is a more general version of the LASSO oracle inequality. It is an adaption of Proposition 1 in \citet{bastani15}, where we plug-in specific values for $\lambda$ and the $\ell_1$-error bound for $\hat{\beta}$. For self-containedness, we also include a proof for Proposition \ref{Prop:1} in the appendix. It allows for the adapted sequence of observations and errors that are $\sigma$-sub-Guassian conditional on all past observations. Moreover, note that the performance of the LASSO estimator dependents on the choice of parameter $\phi$ and the structure of the sample covariance matrix. This is one of the reasons why multi-armed bandit algorithms such as \citet{bastani15} has forced sampling step to ensure that the sample covariance matrix $\hat{\Sigma}_t$ is positive definite in some directions with high probability. 
%Moreover, since the selection rule we use to select the best arm at each time, the generated sequence $\{X_{\hat{a}_t,t}\}_{t=1}^{\infty}$ are highly dependent from each other which is hard to handle. Therefore, we need to apply the martingale concept to construct a uniform confidence set $\mathcal{C}_t$ over time (Corollary \ref{Cor:confidence_set}). 

Secondly, we prove the restricted eigenvalue condition for $\hat{\Sigma}_t$ based on dependent data. Proposition~\ref{Prop:2} shows that with the selection rule induced by our method and the construction of the confidence set, the sample covariance matrix will satisfy the restricted eigenvalue condition with high probability.

\begin{proposition} \label{Prop:2}
Suppose we construct the confidence set $\mathcal{C}_{t+1}=\{\beta\in \mathbb{R}^d: \|\hat{\beta}_t-\beta\|_1\le \frac{6s_0 \sigma x_{\max}}{\phi^2}\sqrt{\frac{2\log t+2\log d}{t}}\}$ with some constant $\phi>0$ that is large enough. Then if the time horizon $t$ exceeds a certain threshold (i.e. $t\ge \mathcal{O}(\log T+s_0K \log K\log d+\frac{\sigma^2 x_{\max}^4s_0^2\zeta^2K^2\log d}{\Delta_*^2})$), the sample covariance matrix $\hat{\Sigma}_t$ are guaranteed to satisfy the compatibility condition with high probability, i.e.
\begin{equation*}
    P[\hat{\Sigma}_t\in \mathcal{C}(supp(\beta_\ast),\phi_\ast)]\ge 1-\mathcal{O}(1/t)-\mathcal{O}(e^{-t+s_0\log K\log d}),
\end{equation*}
where $\phi_\ast=\frac{\phi_0}{8}$.
\end{proposition}
The detailed proof of Proposition \ref{Prop:2} is given in Section \ref{sec:prop2}. Here we only provide some intuition about how we prove Proposition \ref{Prop:2}. Firstly, we provide a crude result of the restricted eigenvalue condition for the sample covariance matrix in Proposition~\ref{Prop:3}, i.e., the restricted eigenvalue condition constant is of order $1/K$ where $K$ is the number of arms. The proof of Proposition~\ref{Prop:3} consists of three main steps:
\begin{enumerate}[label=(\roman*)]
    \item First, we discretize $\mathcal{H}=\{v\in S^{d-1}: \|v\|_0\leq Cs_0\}$, where constant $C=\Theta(\zeta \Lambda_0^2 K\log K)$, with a $\epsilon$-net $\mathcal{N}_{\epsilon}$ (Definition \ref{Def:4}). Then, we prove the restricted eigenvalue condition for $\hat{\Sigma}_t$ over vectors in $\mathcal{N}_{\epsilon}$. In detail, We apply concentration inequalities after properly decoupling the dependence structure induced by the bandit policy to obtain a lower bound of order $1/K$ on the quadratic form $\langle v, \hat \Sigma_t v\rangle$ for all $v\in\mathcal{N}_{\epsilon}$ with high probability. This step is guaranteed by Assumption~\ref{Assump:stochastic}\ref{Assump:anti}, which ensures that the probability of feature vectors locating at the original point is small.
    \item We prove the restricted eigenvalue condition over $\mathcal{H}$ through approximation from $\mathcal{N}_{\epsilon}$.
    \item We prove the eigenvalue condition over the cone $\mathbb C_3(S)=\{v\in\mathbb R^d:\,\|v_{S^c}\|_1\leq 3\|v_S\|_1\}$ (Definition \ref{Def:3}), which implies the restricted eigenvalue condition. Here we apply the Transfer Principle (Lemma \ref{lemma:transfer}) and carefully choose the constant $C$ as discussed earlier in step (i).
\end{enumerate}

Then in Proposition~\ref{Prop:2}, we prove a refined restricted eigenvalue condition based on Proposition~\ref{Prop:3}. The crucial part is to eliminate the dependence on $K$ which is required in Proposition~\ref{Prop:3}. This is challenging since the observed samples are high dependent and in Proposition~\ref{Prop:3}, we need to consider the worst case across $K$ arms. Instead of applying forced sampling to guarantee the eigenvalue condition of $\hat{\Sigma}_t$ as in \citet{bastani15}, we apply a novel technique to prove the result in Proposition \ref{Prop:2}. In particular, with the crude restricted eigenvalue condition in Proposition \ref{Prop:3}, we can show that the optimal arm will be pulled sufficient times, i.e., a positive fraction of time horizon. As a result, the restricted eigenvalue condition in Proposition \ref{Prop:2} is satisfied since Assumption~\ref{Assump:stochastic}\ref{Assump:eigenvalue}, which ensures that the feature vectors are randomly diverse, i.e., the minimum eigenvalue of the population covariance matrix is positive. Intuitively, this can guarantee that the space of feature vectors can be automatically explored without forced sampling. The main proof of Proposition \ref{Prop:2} can be summarized in two steps:
\begin{enumerate}[label=(\roman*)]
    \item In Lemma \ref{lemma:10}, we show that based on Proposition~\ref{Prop:3}, the algorithm will pull the optimal arm at a positive fraction of times after some time point. In particular, Assumption \ref{Assump:reward}\ref{Assump:margin} and \ref{Assump:stochastic}\ref{Assump:eigenvalue} guarantee that $P(\Gamma_t)\geq 1/2$. Under event $\Gamma_t$, there is a positive gap $\Delta_{\ast}$ between the reward of optimal arm and sub-optimal arms. According to Proposition \ref{Prop:1} and Proposition \ref{Prop:3}, for each arm $X_{a,t}$, the optimal reward of $\tilde{\beta}$ within $\mathcal{C}_{t}$ in \eqref{eq:2} is close to the true reward of $X_{a,t}$. Thus, the algorithm will only select the optimal arm under $\Gamma_t$ after some time point according to equation \eqref{eq:2}.
    \item Assumption \ref{Assump:stochastic}\ref{Assump:eigenvalue} guarantees the sparse eigenvalue condition with constant $\phi_0$ for optimal arm under event $\Gamma_t$. We have shown in Lemma \ref{lemma:10} that the optimal arm will be pulled frequently. Then, combining the above results, we prove the restricted eigenvalue condition for $\hat{\Sigma}_t$ with constant $\phi_{\ast}=\phi_0/8$, which does not depend on $K$, even though the observed samples are dependent.
\end{enumerate}

Thirdly, we provide the following corollary which shows that the true parameter $\beta_\ast$ lies in the $\ell_1$-confidence set $\mathcal{C}_{t+1}$ with high probability. Corollary \ref{Cor:confidence_set} follows from a direct application of Proposition \ref{Prop:1} and Proposition \ref{Prop:2}, and will play a crucial role in analyzing the cumulative regret of our algorithm.

\begin{corollary} \label{Cor:confidence_set}
Suppose that Assumptions 1--2 hold and the regularization parameter is chosen as $\lambda_t=2\sigma x_{\max}\sqrt{(2\log t+2\log d)/t}$. Then, there exist absolute constants $C, c_1, c_2>0$, such that for all $t\ge C\,(\log T+s_0K \log  K\log d+\frac{\sigma^2 x_{\max}^4s_0^2\zeta^2 K^2\log d}{\Delta_*^2})$, with probability at least $1-c_1t^{-1}-c_2 e^{-t+s_0\log  K\log d}$, $\beta_\ast$ lies in the confidence set $\mathcal{C}_{t+1}$ defined by Equation \eqref{eq:5}. 

%\begin{equation*}
 %   \mathcal{C}_{t+1}=\left\{\beta\in \mathbb{R}^d: \|\hat{\beta}_t-\beta\|_1\le \frac{192s_0 \sigma x_{\max}}{\phi_0^2}\sqrt{\frac{2\log t+2\log d}{t}}\right\}.
%\end{equation*}
%\vspace{-1em}
\end{corollary}

{The constants $c_1, c_2$ in Corollary \ref{Cor:confidence_set} are absolute constants and can be chosen as $c_1={1}/({2e})$ and $c_2=2e^2$.} The size of the confidence set $\mathcal{C}_{t+1}$ scales as $\mathcal{O}(\sqrt{\frac{\log d +\log t}{t}})$ in dimension $d$ and time $t$. In comparison, the size of the ellipsoid confidence set of \citet{abbasi11} centered at the ridge regression estimator scales as $\mathcal{O}(\sqrt{\frac{d\log t}{t}})$, which has a polynomial dependence on the dimension $d$, and is exponentially larger than the size of our $\ell_1$ confidence set. 

Finally, we apply the result in Corollary \ref{Cor:confidence_set} to compute the cumulative regret of the proposed algorithm. As discussed after Theorem \ref{Thm:2}, for the instant regret incurred in ``burn-in" period or when $C_{t}$ does not contain $\beta_{\ast}$, we simply bound it with worst-case regret $2bx_{\max}$. When $C_t$ contains $\beta_{\ast}$, we consider two complementary cases based on event $\Gamma_{\xi_t}=\{\omega\in \Omega: \langle X_{a^*_t,t}, \theta_*\rangle\ge \langle X_{b,t},\theta_*\rangle +\xi_t, \; \forall b\neq a^*_t\}$ with $\xi_t=2x_{\max}\tau_t$. We show that when $\Gamma_{\xi_t}$ holds, the algorithm will only select optimal arm and incurs no regret. Then according to Assumption~\ref{Assump:reward}\ref{Assump:margin}, we bound the regret incurred when $\Gamma_t$ does not hold. In the end, we sum up the regret from different parts to obtain the cumulative regret.

\section{Experimental Results}
\label{Sec:experiment}

We compare our $\ell_1$-confidence ball based method with (i) the OFUL-LS method proposed in \citet{abbasi11}, (ii) the OLS-bandit introduced in \citet{golden}, and (iii) the LASSO-bandit algorithm in \citet{bastani15} in both synthetic data and real data experiments. The first two methods are not specifically designed for high-dimensional settings. We note that the parametrization of OLS-bandit and LASSO-bandit methods is slightly different from ours, because they treat different arms to have different parameter vectors and one common feature vector. We  apply these methods after converting their parametrization into ours.

\subsection{Synthetic Data}

In the synthetic data experiment, we consider three scenarios for $K$, $d$ and $s_0$: (1) $K=5$, $d=100$, $s_0=5$; (2) $K=5$, $d=1000$, $s_0=5$; (3) $K=50$, $d=20$, $s_0=2$. In each case, a randomly chosen subset of $s_0$ features is predictive of the reward. 
%, i.e. the parameter $\beta_\ast$ is set to zero except $s_0$ randomly selected components that are drawn from a uniform distribution. We consider two scenarios for the region of the uniform distribution:(1) $[0,1]$; (2) $[0,0.2]$. 
At each time $t$, feature vectors for $K$ arms are i.i.d.~ generated by truncating a Gaussian distribution $N(0,\Sigma)$ with $\Sigma_{ij}=0.5^{|i-j|}$ so that $\|X_{a,t}\|_{\infty}\leq 1$. The error term $\epsilon_{a,t}$ follows a zero mean normal distribution with variance $\sigma^2=1$. For this setting, our margin condition holds with $\alpha = 1$ so that we expect a regret growing logarithmically in $T$. 
%All bandit algorithms need the decision-makers to specify some input parameters which are unknown in practice. For example, Corollary \ref{Cor:confidence_set} shows that we need to specify input parameters such as $\sigma$ and $\phi_0$, and LASSO bandit algorithm needs to specify $\sigma$ also. Therefore, we will make some choices for the input parameters of our method, and use the suggested input parameters for the LASSO-bandit, OLS-bandit and OFUL-LS methods. Later, we change the input parameters for our method to check the robustness of its performance. 
For our method, we choose the initial regularization parameter $\lambda_0=0.5$ and diameter $\tau_0=1$ and the method is robust to the choice of these tuning parameters. For LASSO and OLS, we choose the forced sampling parameter $q=1$, the localization parameter $h=5$ for LASSO-bandit and $h=1$ for OLS-bandit. For LASSO-bandit, we further set the initial parameters $\lambda_1=\lambda_{2,0}=0.5$ and for OFUL-LS, we set $\lambda=1$ and $\delta=10^{-4}$. %These tuning parameters for the competitors are optimized to deliver the best results. %{\color{red} What do we expect the $\alpha$ to be here?}

Figure \ref{fig:1} compares our method with competitors on the aforementioned synthetic data with a time horizon of $T=2000$. The curves are the average cumulative regrets over $5$ trials. We observe that the proposed $\ell_1$-confidence ball based algorithm outperforms all the three competing algorithms in all the cases. 
% while $\ell_1$-confidence ball algorithm can automatically encourage exploration via the use of confidence balls. 
%Moreover, it takes longer before the trend of the curves becomes flat in Figure \ref{fig:2}, which is reasonable since algorithms need more samples to estimate parameters when the magnitude of parameters is small. 
In cases (1) and (2) where dimension $d$ is from moderately large to high, OLS-bandit and OFUL-LS algorithms do not perform well since they are not designed for high-dimensional settings and fail to capture the sparsity structure. In case (3) where $K$ is large and the feature vector is low-dimensional, our $\ell_1$-confidence ball algorithm still outperforms the competitors, while the performance of LASSO-bandit is no longer competitive. One possible reason is that since the number of arms is larger, it may need much more exploration in the feature space than it is allowed. LASSO-bandit  does forced sampling only at a limited number of pre-specified times due to which it may not have sufficient exploration to accommodate the large number of arms. In contrast, our algorithm performs an implicit exploration and does not require forced sampling. 

\begin{figure}[h]
    \centering
    % \vspace{.3in}
    \begin{subfigure}[b]{0.3\textwidth}
        \includegraphics[width=\textwidth]{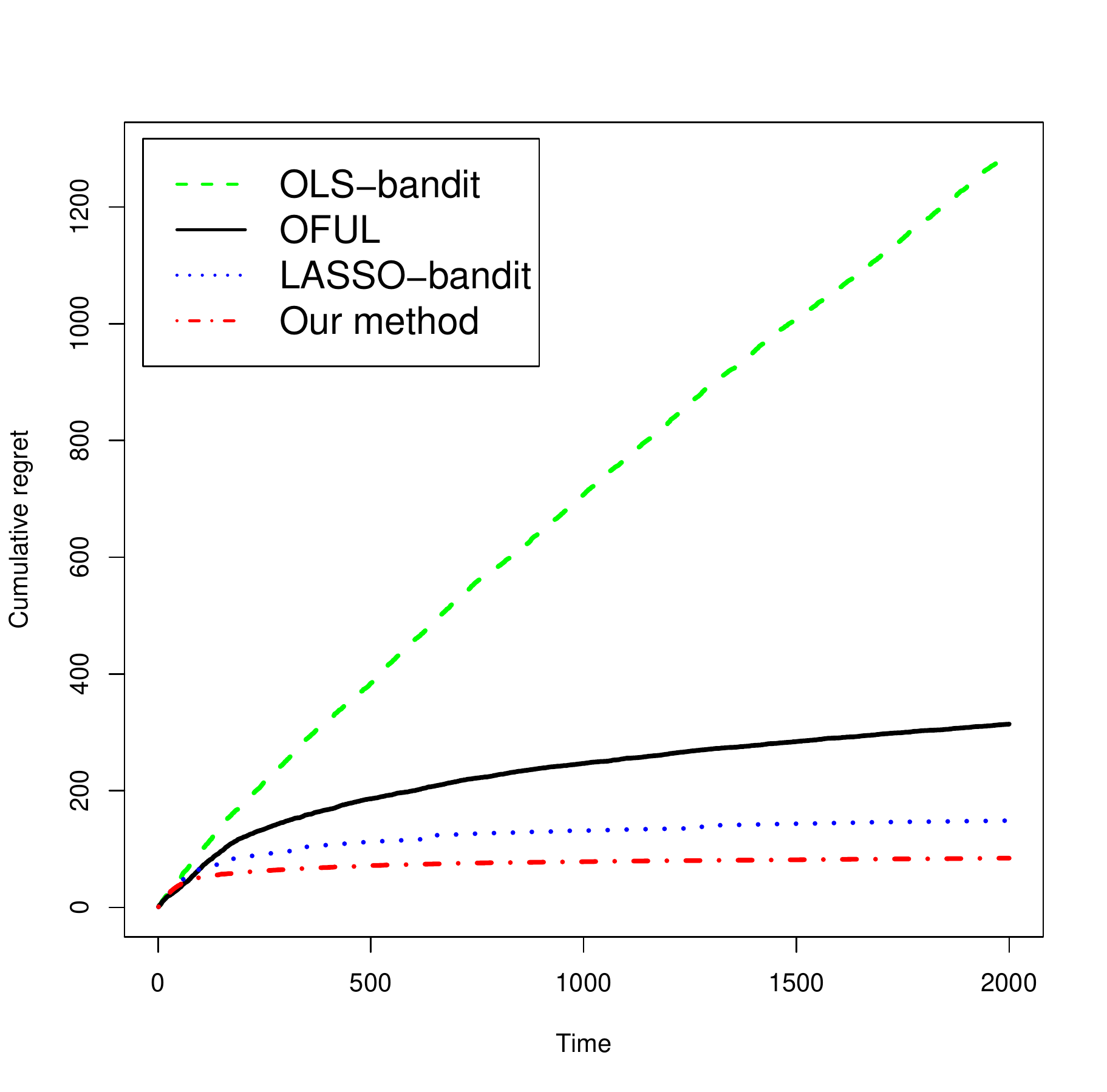}
        \caption{$K=5; d=100; s_0=5$}
        \label{fig:case1_large}
    \end{subfigure}
    ~ %add desired spacing between images, e. g. ~, \quad, \qquad, \hfill etc. 
    \begin{subfigure}[b]{0.3\textwidth}
        \includegraphics[width=\textwidth]{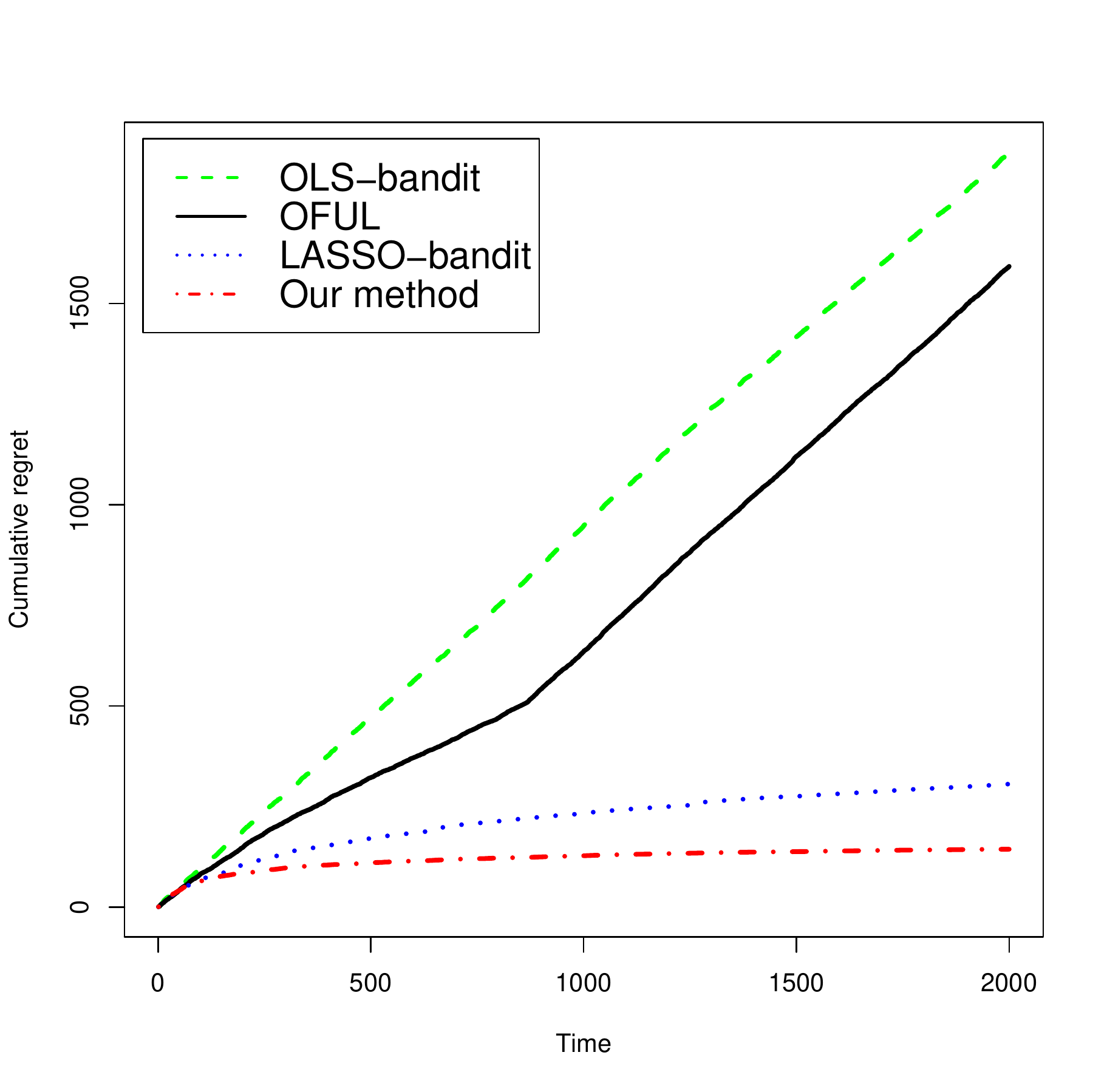}
        \caption{$K=5; d=1000; s_0=5$}
        \label{fig:case2_large}
    \end{subfigure}
    ~ %add desired spacing between images, e. g. ~, \quad, \qquad, \hfill etc. 
    \begin{subfigure}[b]{0.3\textwidth}
        \includegraphics[width=\textwidth]{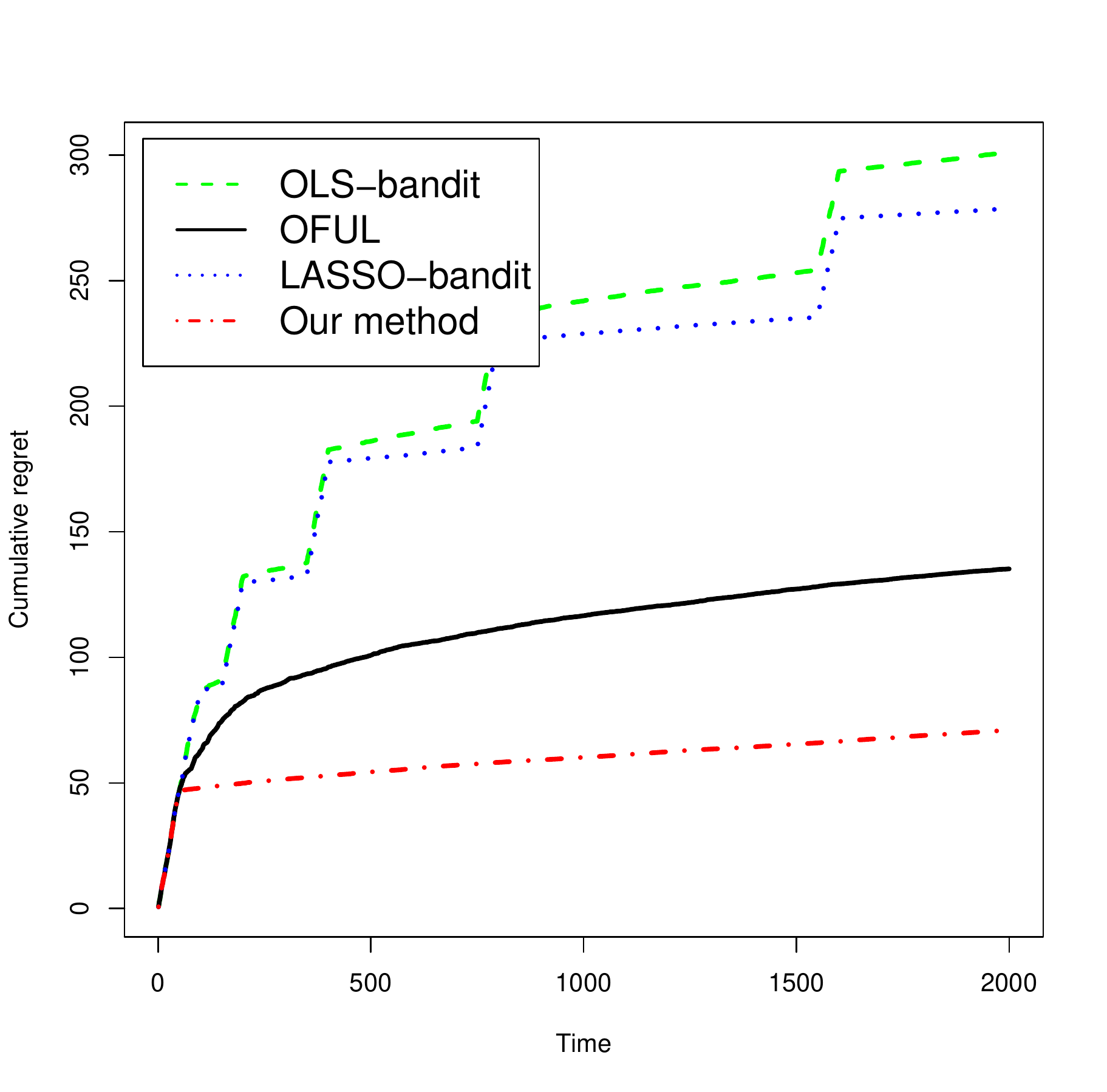}
        \caption{$K=50; d=20; s_0=2$}
        \label{fig:case3_large}
    \end{subfigure}
  \\
  \centering
    \begin{subfigure}[b]{0.3\textwidth}
        \includegraphics[width=\textwidth]{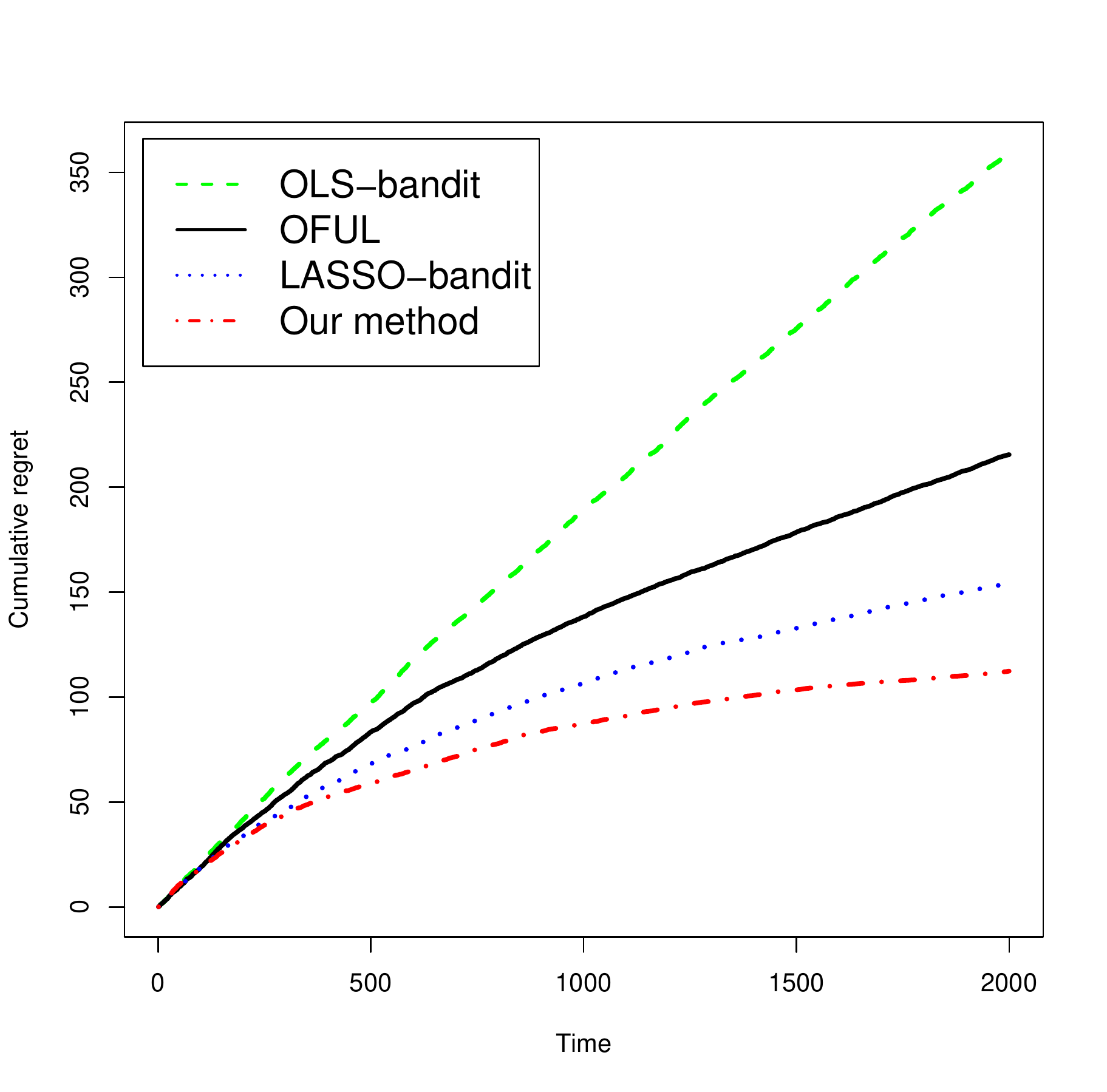}
        \caption{$K=5; d=100; s_0=5$}
        \label{fig:case1_small}
    \end{subfigure}
    ~ %add desired spacing between images, e. g. ~, \quad, \qquad, \hfill etc. 
    \begin{subfigure}[b]{0.3\textwidth}            \includegraphics[width=\textwidth]{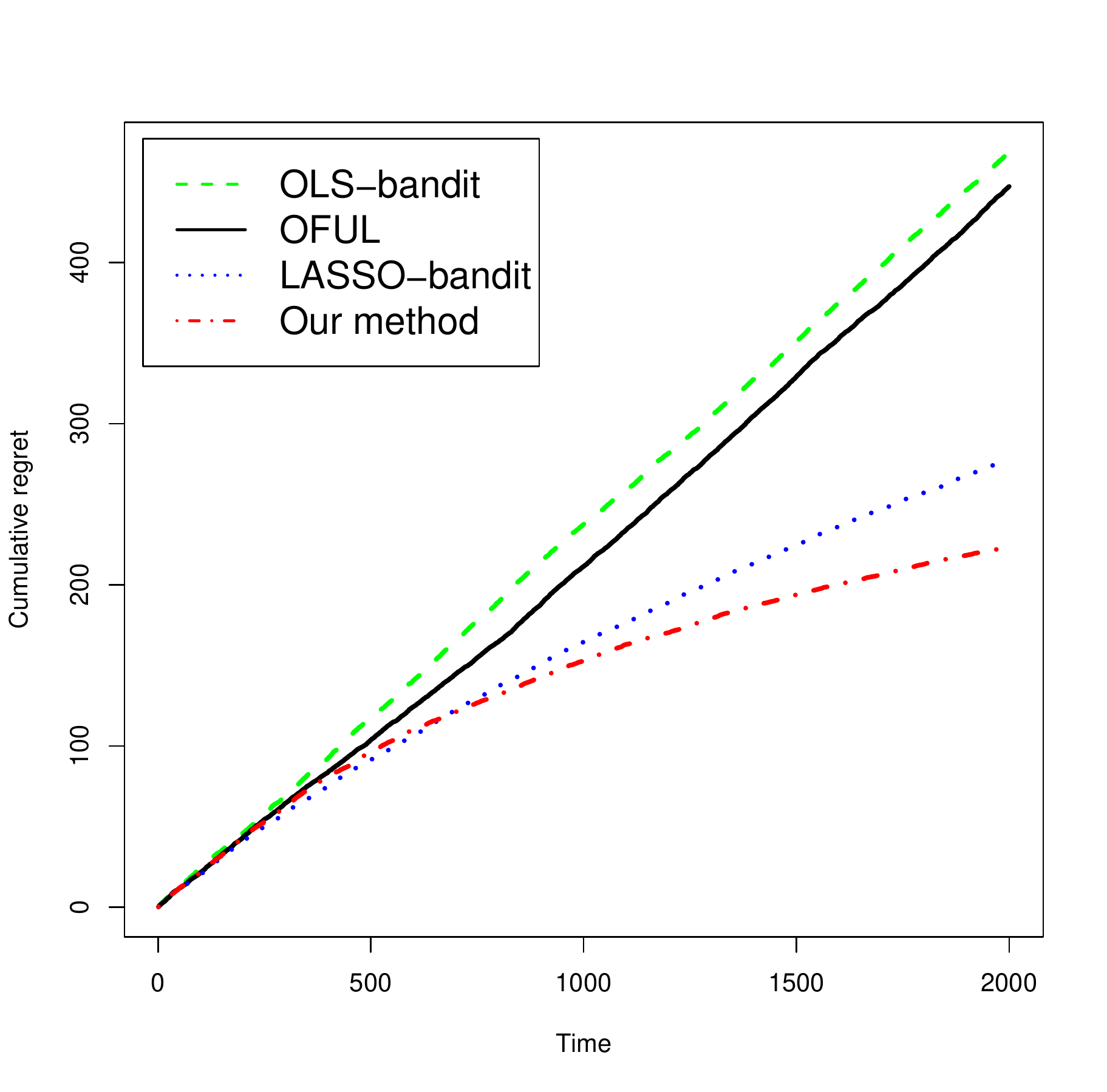}
        \caption{$K=5; d=1000; s_0=5$}
        \label{fig:case2_small}
    \end{subfigure}
    ~ %add desired spacing between images, e. g. ~, \quad, \qquad, \hfill etc. 
    \begin{subfigure}[b]{0.3\textwidth}
        \includegraphics[width=\textwidth]{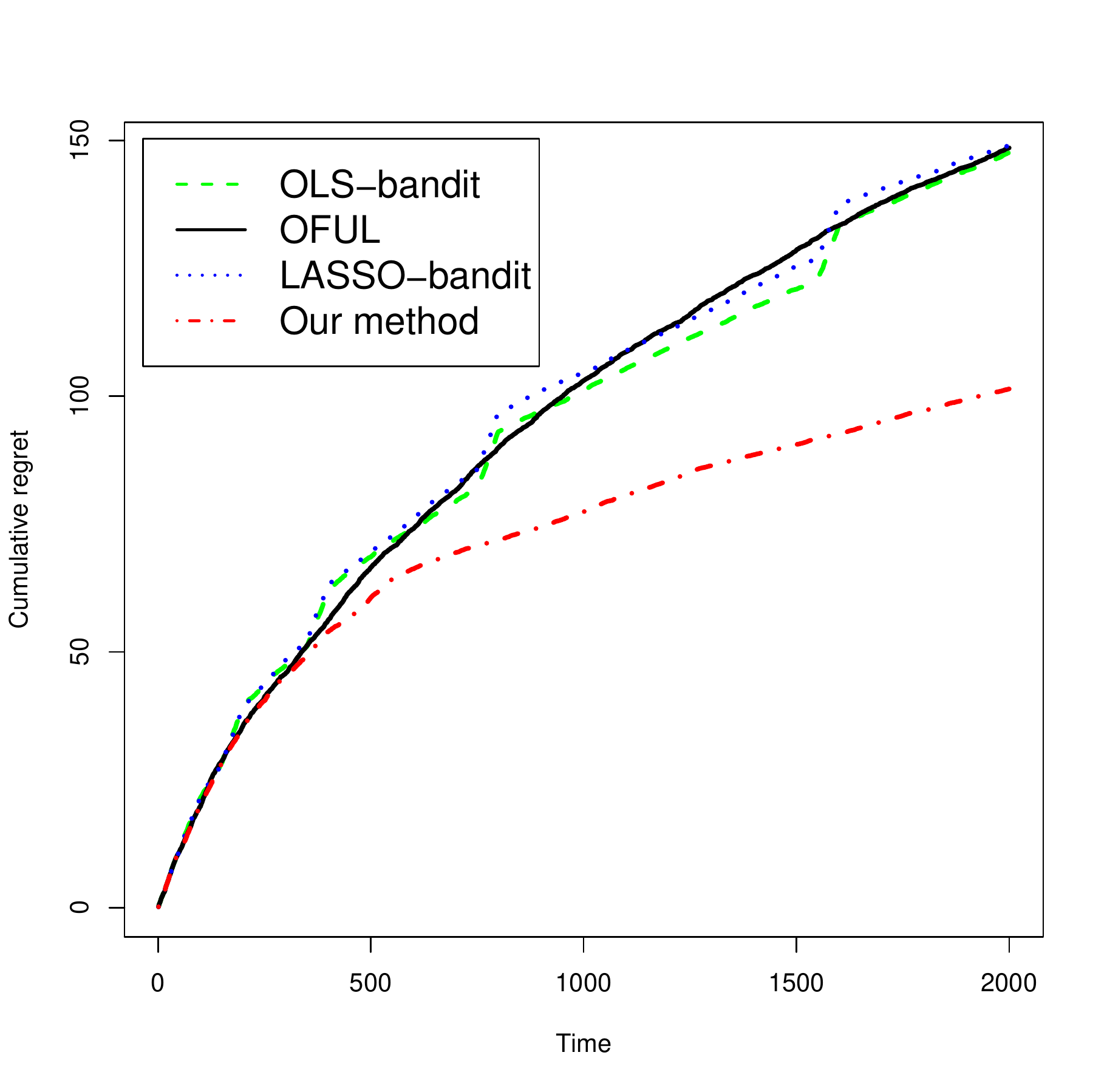}
        \caption{$K=50; d=20; s_0=2$}
        \label{fig:case3_small}
    \end{subfigure}
    % \vspace{.3in}
    \caption{Comparison of the cumulative regret of the $\ell_1$-confidence ball algorithm against competitors. Nonzero components in $\beta_\ast$ are generated from Unif$(0,1)$ in (a)-(c), and Unif$(0,0.2)$ in (d)-(f). Lower value indicates better performance.}\label{fig:1}
\end{figure}

\subsection{Warfarin Dosing Data}
We now consider data from a real experiment from the healthcare context where a physician needs to determine the optimal warfarin dosage for each patient. The warfarin dataset \cite{warfarin} has experimental data on 5528 patients, and contains information on patients' demographics, diagnosis and genetic information. The same dataset is also used by \citet{bastani15}, where they demonstrated the benefits of applying LASSO-bandit algorithms over OFUL-LS, OLS-bandit algorithms, and other low-dimensional methods. 
\citet{bastani15} partitioned the dataset into three arms based on the optimal dosage for each patient. However, their partitions are highly unbalanced with one arm having only $13\%$ and another arm having $54\%$ of the patients. Therefore, in our analysis, we evenly partition the dataset into four arms based on the quantiles of the optimal warfarin dosages in the dataset: (1) Level 1: $(0.0, 19.5]$ mg/week, (2) Level 2: $(19.5, 28.0]$ mg/week, (3): (28.0, 38.5] mg/week, and (4) Level 4: $(38.5, \infty)$ mg/week. The proportion of each arm is nearly $25\%$ after this partition. Following \citet{bastani15}, we construct $94$ patient-specific covariates, including intercept and indicators for categorical variables.

We apply our $\ell_1$-confidence ball based method to the dataset along with other methods. We include a Doctor's policy for comparison, which always assigns a level 2 dose that has the highest percentage ($27.9\%$) of patients in the warfarin data. We also include oracle policies that assign the optimal dose given the true parameter $\beta_\ast$. Similar to \citet{bastani15}, the ``true" parameter vector is estimated using all patient outcomes. We consider two versions of optimal policy: Linear Oracle that estimates $\beta_\ast$ using linear regression, and Logit Oracle that estimates $\beta_\ast$ using multinomial logistic regression (with four categories).

For the real dataset, the expected regret \eqref{eq:1} is not computable, since we do not know the true parameter vector $\beta_\ast$. Instead, we use a surrogate measure which calculates the misclassification rate of assigning optimal dosage to patients. The misclassification rate is calculated as the number of incorrect decisions divided by the number of patients. The lower the expected regret, the lower the misclassification rate. We consider $10$ random permutations of patients and take the average of the misclassification rate of $10$ permutations. Figure \ref{fig:warfarin} illustrates the average fraction of incorrect dosing decisions under different policies. We observe that our method outperforms other competitors except for the oracle polices. Especially, when the number of observations is smaller than 1000, the performance of our method is clearly better than other algorithms, which indicates the benefits of our method when the sample size is relatively small compared to the dimension.

\begin{figure}[h]
\vskip 0.2in
\begin{center}
\centerline{\includegraphics[width=0.8\textwidth]{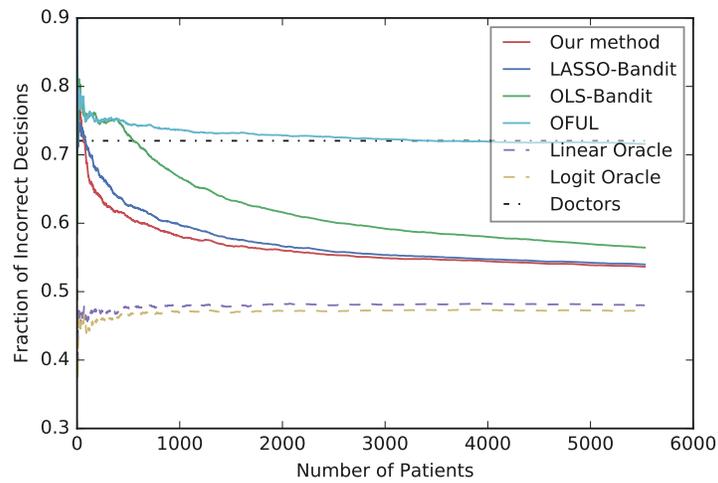}}
\caption{Comparison of the fraction of incorrect dosing decisions under the oracle, LASSO Bandit, OLS Bandit, OFUL,  the $\ell_1$-confidence ball algorithm, and doctor policies in the warfarin data. Lower value indicates better performance. \vspace{4ex}}
\label{fig:warfarin}
\end{center}
\vskip -0.2in
\end{figure}

\newpage
\appendix
\section{Proof of Theorem \ref{Thm:3}}

We will apply the Bayesian decision theory to show the lower bound of the cumulative regret up to horizon time $T$. To prove the $\log T$ term in Theorem \ref{Thm:3}, we can follow the procedure in \citet{golden}. For simplicity, we will focus on the terms in lower bound involving the dimension $d$.

Suppose we have two arms, the first of which is $X^{(1)}=(X_0,X_1, \ldots, X_d)$ and the second is $X^{(2)} = (0, -X_1,\ldots, -X_d)$. Here the two arms are not independent from each other. Entry $X_0$ in the arm vector follows a discrete distribution which will be specified later. The other entries $\{X_1,\ldots, X_d\}$ are i.i.d. truncated normal distribution of $N(0,1)$. For simplicity, we first consider the standard normal case, later the proof can be easily applied to the truncated normal case. The common parameter vector is $\beta$ is from a set $\mathcal{M}=\{\beta\in \mathbb{R}^{d+1}: \beta_0=1,\beta_u=\theta, \beta_j=0 \; \mathrm{for}~j \neq u, u\in[d]\}$. The sparsity of parameters in set $\mathcal{M}$ is then $s_0=2$. Within the set $\mathcal{M}$, the first entry $\beta_0$ of the parameter vector is assumed to be known, so we do not need to estimate it. The remaining $d$ entries of the parameter vector has exactly one nonzero entry with value $\theta$. 
We assume that the parameter vector $\beta$ is uniformly distributed within the set $\mathcal{M}$. In order to guarantee that the assumptions in section \ref{Sec:Assum} hold for the configuration of the parameters in $\mathcal{M}$ and data, we define $\beta_{\min}:=\sqrt{\frac{\log d}{T}}$ and $\theta:=c\beta_{\min}$, where the constant $c$ is sufficiently small and will be defined momentarily, and 
\begin{align}
    \begin{cases}P(X_0 = 0) = C\beta_{\min}^{\alpha} \\
    P(X_0 = a) = \frac{1}{2}(1- C\beta_{\min}^{\alpha}), \quad \text{when}\; a = \pm 1
    \end{cases}
\end{align}
Define $\mathcal{P}_{\alpha}$ as the bandit environment where Assumptions~\ref{Assump:reward}--\ref{Assump:stochastic} hold with constant $\alpha$ in Assumption~\ref{Assump:reward}\ref{Assump:margin}, and $\mathcal{P}_{\mathcal{M}}$ be the set of environments defined as above with $\beta$ from $\mathcal{M}$. We have $\mathcal{P}_{\mathcal{M}}\subseteq \mathcal{P}_{\alpha}$ according to the configuration of $\mathcal{P}_{\mathcal{M}}$. Moreover, we define $\mathcal{F}_{t-1}^+=\sigma(X^{(1)}_{1},X^{(2)}_{1}, Y^{(\pi_1)}_1,\ldots, X^{(1)}_{t-1},X^{(2)}_{t-1}, Y^{(\pi_t)}_{t-1},X^{(1)}_t,X^{(2)}_t)$ as a $\sigma$-algebra.
Then for any policy $\pi$, the supreme of the cumulative regret of $\pi$ at horizon $T$ within $\mathcal{P}_{\alpha}$ can be lower bounded by:
\begin{align}
    & \sup \left\{R_{T}\left(\pi, \pi^{*}\right) : P_{X,Y} \in \mathcal{P}_{\alpha}\right\} \nonumber \\ 
    \geq & \sup _{P_{X,Y}\in\mathcal{P}_{\mathcal{M}}} \mathbb{E} \sum_{t=1}^{T}|\beta^T\Delta_{x,t}|\left[I\left\{\beta^T\Delta_{x,t} \geq 0, \pi_{t}=2\right\}+I\left\{\beta^T\Delta_{x,t}<0, \pi_{t}=1\right\}\right] \nonumber \\ 
    \geq & \mathbb{E}\bigg\{ \sum_{t=1}^{T}\big(\mathbb{E}_{\beta}\left[\beta^T\Delta_{x,t} I\left\{\beta^T\Delta_{x,t} \geq 0\right\} | \mathcal{F}_{t-1}^{+}\right] I\left\{\pi_{t}=2\right\} \nonumber \\
    &~~~~~~~~~~~-\mathbb{E}_{\beta}\left[\beta^T\Delta_{x,t} I\left\{\beta^T\Delta_{x,t}<0\right\} | \mathcal{F}_{t-1}^{+}\right] I\left\{\pi_{t}=1\right\}\big)\bigg\},
\end{align}
where $P_{X,Y}:=\left(P_{X^{(1)}, Y^{(1)}}, P_{X^{(2)}, Y^{(2)}}\right)$ is the joint distribution of the feature vectors and the response variables, and $\Delta_{x,t}:=X^{(1)}_t-X^{(2)}_t=(\Delta_{0,t},\ldots, \Delta_{d,t})$. Here $\mathbb{E}_{\beta}\left[\cdot|\mathcal{F}_{t-1}^+\right]$ denotes the expectation of $\beta$ in $\mathcal{M}$ conditioned on $\mathcal{F}_{t-1}^+$, which is the $\sigma$-algebra generated by feature vectors up to time $t$ and observed response variables of the chosed arms up to time $t-1$. Then, according to the Bayesian decision rule in \citet{golden}, the optimal policy is $\hat\pi_t=1$ if
\begin{align} \label{eq:decision rule}
     \mathbb{E}_{\beta}\left[\beta^T\Delta_{x,t} I\left\{\beta^T\Delta_{x,t} \geq 0\right\} | \mathcal{F}_{t-1}^{+}\right] & \geq - \mathbb{E}_{\beta}\left[\beta^T\Delta_{x,t} I\left\{\beta^T\Delta_{x,t}<0\right\} | \mathcal{F}_{t-1}^{+}\right] \nonumber \\
    \iff \qquad \mathbb{E}_{\beta}\left[\beta^T\Delta_{x,t} | \mathcal{F}_{t-1}^{+}\right] & \geq 0
\end{align}
and $\hat\pi_t=2$ otherwise. Define $\mathcal{F}_{t-1}=\sigma(X^{(1)}_{1},X^{(2)}_{1}, Y^{(\pi_1)}_1,\ldots, X^{(1)}_{t-1},X^{(2)}_{t-1}, Y^{(\pi_t)}_{t-1})$ as the $\sigma$-algebra generated by feature vectors and observed responses up to time $t-1$. We set $\hat\beta_{t-1}=\mathbb{E}\left[\beta|\mathcal{F}_{t-1}^{+}\right]$. Since $\Delta_{x,t}$ is independent from $\mathcal{F}_{t-1}$ and $\beta$, we have $\mathbb{E}\left[\beta|\mathcal{F}_{t-1}^{+}\right]=\mathbb{E}\left[\beta|\mathcal{F}_{t-1}\right]$. The Bayesian policy is then $\hat\pi =\{\hat\pi_t:t\in[T]\}$, where $\hat\pi_t=I\left\{\hat\beta_{t-1}^T\Delta_{x,t} \geq 0\right\}+2I\left\{\hat\beta_{t-1}^T\Delta_{x,t}< 0\right\}$. If we define $u\in [d]$ as the location of the nonzero entry $\theta$ in parameter vector $\beta$, then according to the distribution of $\beta$ within $\mathcal{M}$, $u$ is uniform distributed in $[d]$. Then, we have
\begin{align} \label{eq:c4}
    & \sup \left\{R_{T}\left(\pi, \pi^{*}\right) : P_{X,Y} \in \mathcal{P}_{\alpha}\right\} \nonumber \\
    \geq & \sup \left\{R_{T}\left(\hat\pi, \pi^{*}\right) : P_{X,Y} \in \mathcal{P}_{\mathcal{M}}\right\} \nonumber \\
    \geq & \sum_{t=1}^T \mathbb{E}_{\beta}\mathbb{E}_{X,Y|\beta}\left[|\beta^T\Delta_{x,t}|I\left\{\mathrm{sign}(\beta^T\Delta_{x,t})\neq \mathrm{sign}(\hat\beta_{t-1}^T\Delta_{x,t})\right\} \right] \nonumber\\
    =& \sum_{t=1}^T \mathbb{E}_{\beta,u}\mathbb{E}\left[|X_{0,t}+2\theta X_{u,t}| I\left\{\mathrm{sign}(\beta^T\Delta_{x,t})\neq \mathrm{sign}(\hat\beta_{t-1}^T\Delta_{x,t})\right\}\right] \nonumber \\
    \geq & \sum_{t=1}^T \mathbb{E}_{\beta,u} P(X_0=0)\mathbb{E}_{X,Y|\beta}\left[|2\theta X_{u,t}|I\left\{\mathrm{sign}(\beta^T\Delta_{x,t})\neq \mathrm{sign}(\hat\beta_{t-1}^T\Delta_{x,t})\right\}\big|X_0=0\right] \nonumber \\
    \geq & \sum_{t=1}^T \mathbb{E}_{\beta,u} P(X_0=0) \frac{1}{4} \mathbb{E}_{X,Y|\beta}\left[|2\theta X_{u,t}|I\left\{\hat\beta_{t-1}^T\Delta_{x,t}<0\right\}\big|X_0=0, \beta^T\Delta_{x,t} \geq 0\right] \nonumber \\
    \geq & \sum_{t=1}^T \mathbb{E}_{\beta,u} \frac{1}{2} Cc \beta_{\min}^{\alpha+1} \mathbb{E}_{X,Y|\beta}\left[I\left\{1<X_{u,t}< 2,\; \hat\beta_{t-1}^T\Delta_{x,t}<0\right\}\big|X_0=0, \beta^T\Delta_{x,t} \geq 0\right] \nonumber \\
    = & \sum_{t=1}^T \frac{1}{2} Cc \beta_{\min}^{\alpha+1} \mathbb{E}_{\beta,u} \left[ P_{X,Y|\beta}\left(1<X_{u,t}< 2,\; \hat\beta_{t-1}^T\Delta_{x,t}<0\right)\right].
\end{align}
The third inequality above is by the conditional expectation with respect to the event $\{X_0=0\}$, and the fourth inequality is by the symmetric distribution of $\beta^T\Delta_{x,t}$ when $X_0=0$. The last inequality is by the fact that when $X_{u,t}\in (1,2)$, then $|\theta X_{u,t}|\geq c\beta_{\min}$.

Then, it suffices to prove the lower bound for the expectation term in the last equation, where
\begin{align} \label{eq:c5}
     & \mathbb{E}_{\beta,u} \left[ P_{X,Y|\beta}\left(1<X_{u,t}< 2,\; \hat\beta_{t-1}^T\Delta_{x,t}<0\right)\right] \nonumber \\
     = & \mathbb{E}_{\beta,u} \left[ P\left(1<X_{u,t}< 2\right) P_{X,Y|\beta}\left(\hat \beta_{u,t-1} X_{u,t} +\hat\beta_{-u,t-1}^TX_{-u,t} < 0\mid 1<X_{u,t}< 2\right)\right] \nonumber \\
     \geq & \tilde{C} P\left(\hat \beta_{u,t-1} X_{u,t} +\hat\beta_{-u,t-1}^TX_{-u,t} < 0 \mid 1 < X_{u,t} <2\right).
\end{align}
Here $\hat\beta_{-u,t-1}\in\mathbb{R}^{d-1}$ is the subvector of $\hat{\beta}_{t-1}$ without the $0$-th and $u$-th entries, and $X_{-u, t} := (X_{1,t}, \cdots, X_{u-1,t}, X_{u+1,t}, \cdots, X_d)^T\in\mathbb{R}^{d-1}$. The probability $\tilde{C} := P(1 < X_{u,t} <2)$ is positive, since $X_{u,t}$ is standard normal by definition.

In order to lower bound the probability in \eqref{eq:c5}, we use the Fano's inequality to reduce the probability to the false discovery rate in a multiple hypothesis testing problem. Specifically, we define a test function $\hat u_{t-1} :=\arg\max_{j\in [d]}|\hat\beta_{j,t-1}|$. We then show that the probability in \eqref{eq:c5} is lower bounded by $P(\hat u_{t-1}\neq u)$. In fact, if event $\{\hat u_{t-1}\neq u\}$ holds, then we have $|\hat\beta_{u,t-1}|^2\le\|\hat\beta_{-u,t-1}\|_2^2 $ by the definition of $\hat u_{t-1}$. In addition, we have
\begin{align}
    & P\left(\hat \beta_{u,t-1} X_{u,t} +\hat\beta_{-u,t-1}^TX_{-u,t} < 0 \mid 1 < X_{u,t} <2,  |\hat\beta_{u,t-1}|^2\le\|\hat\beta_{-u,t-1}\|_2^2 \right) \nonumber \\
    \ge & P\left(|2\hat\beta_{u,t-1}|+ \hat\beta_{-u,t-1}^TX_{-u,t} < 0 \mid 1 < X_{u,t} <2, |\hat\beta_{u,t-1}|^2\le\|\hat\beta_{-u,t-1}\|_2^2\right) \nonumber \\
    = & P\left( 2 + N\left(0, \frac{\|\hat\beta_{-u,t-1}\|_2^2}{|\hat\beta_{u,t-1}|^2}\right)< 0 \mid 1 < X_{u,t} <2, |\hat\beta_{u,t-1}|^2\leq\|\hat\beta_{-u,t-1}\|_2^2\right) \nonumber \\
    \ge & P(N(2,1)< 0) = C'
\end{align}
The last equation is because that $X_{-u,t}$ is independent from $\hat\beta_{t-1}$ and $X_{u,t}$. In addition, the random variable $\hat\beta_{-u,t-1}^TX_{-u,t}$ conditioned on $\hat\beta_{t-1}$ follows normal distribution with mean $0$ and variance $\|\hat\beta_{-u,t-1}\|_2^2$. Then, by applying the above inequalities to eq.~\eqref{eq:c5}, we have
\begin{equation}\label{eq:c7}
    \mathbb{E}_{\beta,u} \left[ P_{X,Y|\beta}\left(1<X_{u,t}< 2,\; \hat\beta_{t-1}^T\Delta_{x,t}<0\right)\right] \ge \tilde{C} C' P(\hat u_{t-1} \neq u).
\end{equation}
It suffices to lower bound $P(\hat u_{t-1} \neq u)$ for each time $t$. We first state a variant of Fano's lower bound in multiple hypothesis testing, the details of which can be found in Chapter 15 of \citet{wainwright}.

\begin{lemma}\label{lem:fano}
Assume that $U$ is uniform on $\mathcal{U}$. For any Markov chain $U\rightarrow (X,Y) \rightarrow \hat U$,
\begin{equation}
    P(\hat U \neq U) \geq 1-\frac{I(X,Y; U) +\log 2}{\log(|\mathcal{U}|)}.
\end{equation}
Here $I(X, Y;U)$ is the mutual information between $(X, Y)$ and $U$, and $|\mathcal{U}|$ is the cardinality of set $\mathcal{U}$.
\end{lemma}

According to the definition of $\mathcal{P}_{\mathcal{M}}$, $u$ is uniform in $[d]$, therefore $\log(|\mathcal{U}|)=\log d$. Then, we prove an upper bound for $I(X_{\{1:t-1\}},Y_{\{1:t-1\}};u)$ to guarantee that the probability $P(\hat u_{t-1}\neq u)$ in \eqref{eq:c7} is bounded away from $0$. Here $X_{\{1:t-1\}}$ and $Y_{\{1:t-1\}}$ represent the set of feature vectors and response variables up to time $t-1$. Based on the chain rule of the mutual information, we have
\begin{equation}\label{eq:c9}
    I(X, Y; u) = I(X; u) + I(Y; u| X) = I(Y; u|X),
\end{equation}
where the last equality is due to the independence between $X$ and the uniformly distributed $u$. Moreover, according to Chapter 15 of \citet{wainwright}, the conditional mutual information of $Y$ and $u$ conditoned on $X$ can be upper bounded by the Kullback-Leibler divergence, i.e., 
\begin{equation}
    I(Y; u|X) \le \frac{1}{|\mathcal{U}|^2}\sum_{j,k=1}^{|\mathcal{U}|} D(P_{\beta^j|X} \| P_{\beta^k|X}),
\end{equation}
where $P_{\beta^j|X}$ is the distribution of $Y$ conditioned on $X$ and the parameter vector $\beta^j\in\mathcal{M}$ corresponding to $u^j\in\mathcal{U}$. The KL-divergence of between two normal distributions, i.e., $P_{\beta^j|X}=N(X\beta^j,\sigma^2)$ and $P_{\beta^k|X}=N(X\beta^k,\sigma^2)$, can be upper bounded as
\begin{align}
    &D(P_{\beta^j|X} \| P_{\beta^j|X}) \nonumber\\
    =& \frac{1}{2\sigma^2}\|X(\beta^j-\beta^k)\|_2^2 = \frac{1}{2\sigma^2}\|(X_j-X_k)\theta\|_2^2 \nonumber \\
    \leq& \frac{1}{2\sigma^2}\theta^2 4x_{\max}^2 
    = \frac{2x^2_{\min}}{\sigma^2}\cdot\frac{c^2\log d}{T}.
\end{align}
Summing up the mutual information up to time $t-1$, and according to eq.~\eqref{eq:c9}, we have that
\begin{align}
    I(X_{\{1:t-1\}},Y_{\{1:t-1\}};u) 
    = I(Y_{\{1:t-1\}};u)
    \leq \sum_{s=1}^{t-1} \frac{2x^2_{\min}}{\sigma^2}\cdot\frac{c^2\log d}{T}
    \leq \frac{2x_{\max}^2c^2\log d}{\sigma^2}.
\end{align}
Define constant $c:=\frac{\sigma}{2x_{\max}}$. Then, if $\log d \geq 4\log 2$, we have that
\begin{align}\label{eq:c13}
    P(\hat u_{t-1} \neq u) & \geq 1- \frac{I(X_{\{1:t-1\}},Y_{\{1:t-1\}};u) + \log 2}{\log|\mathcal{U}|} \nonumber \\
    & \geq 1- \frac{\frac{1}{2}\log d+ \log 2}{\log d} \geq \frac{1}{4}.
\end{align}
By applying eq.~\eqref{eq:c13} and eq.~\eqref{eq:c7} to eq.~\eqref{eq:c4}, it can be derived that
\begin{align}
    & \sup \left\{R_{T}\left(\pi, \pi^{*}\right) : P_{X,Y} \in \mathcal{P}_{\alpha}\right\} \nonumber \\
   \geq & \sum_{t=1}^T \frac{1}{2} Cc \beta_{\min}^{\alpha+1} \mathbb{E}_{\beta,u} \left[ P_{X,Y|\beta}\left(1<X_{u,t}< 2,\; \hat\beta_{t-1}^T\Delta_{x,t}<0\right)\right] \nonumber \\
   \geq & \sum_{t=1}^T \frac{1}{2}Cc\beta_{\min}^{\alpha+1}\tilde{C} C' P(\hat u_{t-1} \neq u) \nonumber \\
   \geq & \frac{\tilde{C} C'Cc}{8} (\frac{\log d}{T})^{\frac{\alpha+1}{2}}T. \nonumber
\end{align}
Therefore, the supreme cumulative regret incurred by any policy $\pi$ can be lower bounded as
\begin{align}
  \sup \left\{R_{T}\left(\pi, \pi^{*}\right) : P_{X,Y} \in \mathcal{P}_{\alpha}\right\} \geq \begin{cases}
 \tilde{C}_L\left(\log d\right) ,\quad \alpha=1\\
 \tilde{C}_L\left((\log d)^{\frac{\alpha+1}{2}}T^{\frac{1-\alpha}{2}}\right), \quad \alpha\in [0,1),
\end{cases}
\end{align}
where $\tilde{C}_L$ is some postive constant.

Combining the regret lower bound $\Omega(\log T)$ from \citet{golden} on the horizon $T$, we prove the regret lower bound in Theorem \ref{Thm:3}.

\section{Proof of Theorem \ref{Thm:2}}
\label{Sec:proof_upper_bound}

In this section, we provide the proof for Theorem \ref{Thm:2}. As we discussed in Section \ref{sec:sketched_proof_upper_bound}, we first prove Proposition \ref{Prop:1} and Proposition \ref{Prop:2}. Then Corollary \ref{Cor:confidence_set} is an application of the above propositions. Lastly, we prove Theorem \ref{Thm:2} by applying the result in Corollary \ref{Cor:confidence_set}.

\subsection{Proof of Proposition \ref{Prop:1}} \label{sec: prop1}

For consistency with Proposition 1, let $X_i$ be the $i^{th}$ row of $X$ and $Y(i)$ be the $i^{th}$ entry of $Y$. The sequence $\{X_i: i=1,\ldots,t\}$ forms an adapted sequence of observations, i.e. $X_t$ may depend on the history $\{X_s, Y(s)\}_{s=1}^{t-1}$. And let $\epsilon\in \mathbb{R}^t$ be the $\sigma$-sub-Guassian errors.

Before proving Proposition \ref{Prop:1}, we first stat the following lemmas for adapted sequences.

\begin{lemma} \label{lemma:1} (Bernstein Concentration).
Let $\{D_k,\mathcal{G}_k\}_{k=1}^{\infty}$ be a martingale difference sequence, and suppose that $D_k$ is $\sigma$-sub-Guassian in an adapted sense, i.e,. for all $\alpha\in \mathbb{R}$, $\mathbb{E}[e^{\alpha D_k}|\mathcal{G}_{k-1}]\le e^{\alpha^2\sigma^2/2}$ almost surely. Then, for all $t\ge 0$, $P[|\sum_{k=1}^n D_k|\ge t]\le 2\exp[-t^2/2n\sigma^2]$. 
\end{lemma}

Proof of Lemma \ref{lemma:1} follows from Theorem 2.3 of \citet{wainwright} when $\alpha_*=\alpha_k=0$ and $\nu_k=\sigma$ for all $k$.

\begin{lemma} \label{lemma:2}
Define the event
\begin{equation*}
    \mathcal{F}(\lambda_0(\gamma))=\{\max_{r\in [d]}(2|\epsilon^T X^{(r)}|/t)\le \lambda_0(\gamma) \},
\end{equation*}
where $X^{(r)}$ is the $r^{th}$ column of $X$ and $\lambda_0(\gamma)=2\sigma x_{\max}\sqrt{(\gamma^2+2\log d)/t}$. Then we have $P[\mathcal{F}(\lambda_0(\gamma))]\ge 1-2\exp[-\gamma^2/2]$.
\end{lemma}

Proof of Lemma \ref{lemma:2} can be found in Lemma EC.2. of \citet{bastani15}.

\begin{lemma}\label{lemma:3}
For any $\lambda_0\in \mathbb{R}^+$, when $\lambda\ge \lambda_0$, on event $\mathcal{F}(\lambda_0)$, we have 
\begin{equation*}
    \|X(\hat{\beta}-\beta_\ast)\|_2^2/t\le 3\lambda\|\hat{\beta}_S-\beta_{*,S}\|_1-\lambda\|\hat{\beta}_{S^c}-\beta_{*,S^c}\|_1.
\end{equation*}
\end{lemma}

\begin{proof}
According to the definition of the LASSO estimator (\ref{eq:4}), we have 
\begin{equation}\label{eq:7}
    \frac{1}{2t}\|Y-X\hat{\beta}\|_2^2+\lambda\|\hat{\beta}\|_1\le \frac{1}{2t}\|Y-X\beta_\ast\|_2^2+\lambda\|\beta_\ast\|_1.
\end{equation}
Since $\lambda\ge \lambda_0$, then if event $\mathcal{F}(\lambda_0)$ holds, we have $\lambda\ge 2\|\epsilon^TX\|_{\infty}/t$. Thus,
\begin{align}
    \frac1{2t}\|X(\hat{\beta}-\beta_\ast)\|_2^2 & \le \frac{1}{t}\|\epsilon^T X\|_{\infty}\|\hat{\beta}-\beta_\ast\|_1+\lambda(\|\beta_\ast\|_1-\|\hat{\beta}\|_1) \nonumber \\
    & \le \frac{\lambda}{2}(\|\hat{\beta}_S-\beta_{*,S}\|_1+\|\hat{\beta}_{S^c}-\beta_{*,S^c}\|_1) + \lambda(\|\beta_{*,S}\|_1-\|\hat{\beta}\|_1) \nonumber\\
    & \le \frac{\lambda}{2}(\|\hat{\beta}_S-\beta_{*,S}\|_1+\|\hat{\beta}_{S^c}-\beta_{*,S^c}\|_1) \nonumber \\
    &~~~~+ \lambda(\|\hat{\beta}_S-\beta_{*,S}\|_1-\|\hat{\beta}_{S^c}-\beta_{*,S^c}\|_1)\nonumber\\
    & = \frac{3}{2}\lambda\|\hat{\beta}_S-\beta_{*,S}\|_1-\frac{\lambda}{2}\|\hat{\beta}_{S^c}-\beta_{*,S^c}\|_1. \label{eq:8}
\end{align}
\end{proof}

Now we can prove Proposition \ref{Prop:1}. From Lemma \ref{lemma:3}, we have
\begin{equation}\label{eq:9}
    \|\hat{\beta}_{S^c}-\beta_{*,S^c}\|_1\le 3\|\hat{\beta}_S-\beta_{*,S}\|_1.
\end{equation}
Then we choose $\gamma=2\log t$ for event $\mathcal{F}(\lambda_0(\gamma))$. By the definition of $\lambda$, we have $\lambda\ge \lambda_0$. Thus, if both events $\mathcal{F}(\lambda_0(\gamma))$ and $\{\hat{\Sigma}(X)\in \mathcal{C}(supp(\beta_\ast),\phi)\}$ hold, we have 
\begin{align}
    \|\hat{\beta}-\beta_\ast\|_1^2 & \le \frac{s_0}{\phi^2}\|X(\hat{\beta}-\beta_\ast)\|_2^2/t  \le \frac{s_0}{\phi^2}(\frac{2}{t}\|\epsilon^T X\|_{\infty}+2\lambda)\|\hat{\beta}-\beta_\ast\|_1 \nonumber\\
    & \le \frac{3s_0\lambda}{\phi^2}\|\hat{\beta}-\beta_\ast\|_1 = \frac{6s_0 \sigma x_{\max}}{\phi^2}\sqrt{\frac{2\log t+2\log d}{t}} \|\hat{\beta}-\beta_\ast\|_1.\label{eq:10}
\end{align}
Thus,
\begin{align}
    & P\left[\|\hat{\beta}-\beta_\ast\|_1\le \frac{6s_0 \sigma x_{\max}}{\phi^2}\sqrt{\frac{2\log t+2\log d}{t}} \right] \nonumber\\
    \ge & P\left[\mathcal{F}(\lambda_0(\gamma))\right]-P\left[\{\hat{\Sigma}(X)\notin \mathcal{C}(supp(\beta_\ast),\phi)\}\right] \nonumber\\
    \ge & 1-\frac{2}{t}-P\left[\{\hat{\Sigma}(X)\notin \mathcal{C}(supp(\beta_\ast),\phi)\}\right]. 
    \label{eq:11}
\end{align}

\subsection{Proof of Proposition \ref{Prop:2}} \label{sec:prop2}

\begin{definition}\label{Def:3}
 For a constant $\alpha \ge 1$ and index set $I$, define the set 
\begin{equation*}
    \mathbb{C}_{\alpha}(I)=\left\{\Delta\in \mathbb{R}^d: \|\Delta_{I^c}\|_1\le \alpha \|\Delta_I\|_1\right\}.
\end{equation*}
\end{definition}

Our goal is to prove that with high probability, for $\forall v\in \mathbb{C}_3(S)$
\begin{equation*}
    v^T \hat{\Sigma}_t v=\frac{1}{t}\|X_t v\|_2^2=\frac{1}{t}\sum_{s=1}^t \langle X_{\hat{a}_s,s},v\rangle^2 \ge \kappa \|v\|^2_2,
\end{equation*}
for some constant $\kappa>0$. To prove Proposition \ref{Prop:2}, we first state a weaker version about the compatibility condition of the sample covariance matrix.

\begin{proposition}\label{Prop:3}
For the adapted sequence $\{X_{\hat{a}_s,s}: s=1,\ldots,t\}$ induced by the bandit policy, the sample covariance matrix $\hat{\Sigma}_t$ are guaranteed to satisfy the compatibility condition uniformly with high probability, i.e.
\begin{equation*}
    P\left(\hat{\Sigma}_t\in \mathcal{C}(supp(\beta_\ast),\frac{1}{\sqrt{32\zeta K}}) \right)\ge 1-\mathcal{O}(e^{s_0 K\log K \log d}[e^{-c_0t}+e^{-\tilde{C}\log(K) t}] ),
\end{equation*}
where $c_0,\tilde{C}$ are constants. In addition, we can derive a uniform bound for the compatibility condition over $t$ that exceeds a certain threshold (i.e. $t\ge \mathcal{O}(\log T+s_0K\log K \log d)$) with high probability, i.e.
\begin{equation*}
    P\left(\forall t\ge T', \hat{\Sigma}_t\in \mathcal{C}(supp(\beta_\ast),\frac{1}{\sqrt{32\zeta K}})\right)\ge 1-\frac{2}{T},
\end{equation*}
where $T'=\mathcal{O}(\log T+s_0K \log K \log d)$.
\end{proposition}

We first provide the outline of the proof for Proposition \ref{Prop:3}.
\begin{enumerate}[label=(\roman*)]
    \item Discretize the unit sphere $S^{d-1}\in\mathbb{R}^d$, and show eigenvalue condition of $\hat\Sigma_t$ on a finite set $\mathcal{N}_0$.
    \item Show eigenvalue condition for all $m$-sparse vectors within sphere $S^{d-1}$.
    \item Transfer eigenvalue condition of $\hat\Sigma_t$ to vectors in $\mathbb{C}_3(S)$, which implies compatibility condition.
\end{enumerate}
In the following, we state a result from \citet{oliveira}, which is used for step (iii) in the proof.

\begin{lemma}[Transfer Principle]\label{lemma:transfer}
Suppose $\hat\Sigma_t$ and $\Sigma$ are matrix with non-negative diagonal entries, and assume $\eta\in(0,1)$, $m\in\{1,\ldots, d\}$ are such that
\begin{equation} \label{eq:restricted_eigen}
    \forall v \in \mathbb{R}^d \text { with }\|v\|_0 \leq m, v^{T} \widehat{\Sigma}_t v \geq(1-\eta) v^{T} \Sigma v.
\end{equation}
Assume $D$ is a diagonal matrix whose elements $D_{j,j}$ are non-negative and satisfy $D_{j,j}\geq [\hat\Sigma_{t}]_{j,j}-(1-\eta)\Sigma_{j,j}$. Then 
\begin{equation}\label{eq:transfer}
    \forall x \in \mathbb{R}^d, x^{T} \widehat{\Sigma}_t x \geq(1-\eta) x^{T} \Sigma x-\frac{\|D^{1 / 2} x\|_1^2}{m-1}.
\end{equation}
\end{lemma}
The proof of Lemma \ref{lemma:transfer} can be found in Lemma 5.1 of \citet{oliveira}.

According to Lemma \ref{lemma:transfer}, it suffices to prove the eigenvalue condition \eqref{eq:restricted_eigen} for $m$-sparse vectors with sufficient large $m=Cs_0$, where constant $C$ will be specified later. To prove this, we first introduce $\epsilon$-net on unit sphere for approximation.

\begin{definition}\label{Def:4}
 (Nets, covering numbers). Let $(X,d)$ be a metric space and let $\epsilon>0$. A subset $\mathcal{N}_{\epsilon}$ of $X$ is called an $\epsilon$-net of $X$ if every point $x\in X$ can be approximated to within $\epsilon$ by some point $y\in \mathcal{N}_{\epsilon}$, i.e. so that $d(x,y)\le \epsilon$. The minimal cardinality of an $\epsilon$-net of $X$, if finite, is denoted $\mathcal{N}(X,\epsilon)$ and is called the covering number of $X$ (at scale of $\epsilon$).
\end{definition}

The following provides a bound for the cardinality of the $\epsilon$-net that can approximate the points on the unit Euclidean sphere within range of $\epsilon$.

\begin{lemma}\label{lemma:5}
(Covering numbers of the sphere). The unit Euclidean sphere $S^{d-1}$ equipped with the Euclidean metric satisfies for every $\epsilon>0$ that 
\begin{equation*}
    \mathcal{N}(S^{d-1}, \epsilon)\le \left(1+\frac{2}{\epsilon}\right)^d.
\end{equation*}
\end{lemma}

% Denote $E_J=\mathrm{span}\{e_j: j\in J\}$, then we can find an approximation net for all sparse vectors in $\mathbb{C}_{\alpha}(S)$.

% \begin{lemma}\label{lemma:6}
%  Suppose we have a $\frac{2(\alpha+1)}{\sqrt{C}}$-net for $S^{d-1}\cap (\cup_{|J|=Cs_0} E_J)$, then for $\forall v\in \mathbb{C}_{\alpha}(S)\cap S^{d-1}$, there exists a vector $u$ in the aforementioned net such that $\|u-v\|_2\le \frac{6(\alpha+1)}{\sqrt{C}}$.
% \end{lemma}

% \begin{proof}
% By Lemma \ref{lemma:1}, we know that there exists a $Cs_0$-sparse vector $\tilde{u}$ such that $\|\tilde{u}-v\|_2\le\frac{2(\alpha+1)}{\sqrt{C}}$ since $\|v\|_1\le (1+\alpha)\|v_S\|_1\le (1+\alpha)\sqrt{s_0}\|v_S\|_2\le (1+\alpha)\sqrt{s_0}$. And for this vector $\tilde{u}$, we can find a vector $u$ in the net such that $\|u-\frac{\tilde{u}}{\|\tilde{u}\|_2}\|_2\le \frac{2(\alpha+1)}{\sqrt{C}}$. Then we have
% \begin{align*}
%     \|u-v\|_2&\le \|u-\frac{\tilde{u}}{\|\tilde{u}\|_2}\|_2+ \|\tilde{u}-\frac{\tilde{u}}{\|\tilde{u}\|_2}\|_2+\|\tilde{u}-v\|_2\\
%     &\le \frac{4(\alpha+1)}{\sqrt{C}}+|\|\tilde{u}\|_2-1|\le \frac{6(\alpha+1)}{\sqrt{C}}.
% \end{align*}
% \end{proof}

% Define $\mathcal{N}_0$ as the $\frac{2(\alpha+1)}{\sqrt{C}}$-net for $S^{d-1}\cap (\cup_{|J|=Cs_0} E_J)$. Then from Lemma \ref{lemma:5}, we have a bound for the covering number $N$ of $\mathcal{N}_0$.
% \begin{align*}
%     \log N&=\log\left(\binom{d}{Cs_0}(1+\frac{2}{\frac{2(\alpha+1)}{\sqrt{C}}})^{Cs_0}\right)\le \log\left((\frac{ed}{Cs_0})^{Cs_0}(\frac{3\sqrt{C}}{2(\alpha+1)})^{Cs_0}\right)\lesssim Cs_0\log d.
% \end{align*}

Denote $E_J=\mathrm{span}\{e_j: j\in J\}$, then we define $\mathcal{N}_0$ as the $\epsilon$-net for $S^{d-1}\cap (\cup_{|J|=Cs_0} E_J)$. For each subset $J\subseteq [d]$ with $|J|=Cs_0$, the set $S^{d-1}\cap E_J$ can be viewed as a unit sphere in $\mathbb{R}^{Cs_0}$. Then according to Lemma \ref{lemma:5}, we have a bound for the covering number $N$ of $\mathcal{N}_0$ when $\epsilon\leq 1$.
\begin{align}\label{eq:net_card}
    \log N &=\log\left(\binom{d}{Cs_0}\left(1+\frac{2}{\epsilon}\right)^{Cs_0}\right) \nonumber \\
    &\leq \log\left(\left(\frac{ed}{Cs_0}\right)^{Cs_0}\left(\frac{3}{\epsilon}\right)^{Cs_0}\right) \nonumber \\
    &\lesssim Cs_0[\log d + \log(\frac{1}{\epsilon})].
\end{align}

According to Assumption \ref{Assump:anti}, when $h>0$ is small, i.e. $h=\mathcal{O}(\frac{1}{K})$, we have for $\forall v\in\mathbb{R}^d$ that:
\begin{equation*}
    P(\min_a\langle X_{a,t},v\rangle^2 \le h\|v\|_2^2)\le \zeta Kh.
\end{equation*}

From now on we fix a sparse vector $v$ with $\|v\|_2=1$ and $\|v\|_0\le Cs_0$. Define 
\begin{equation*}
    N_t=\sum_{s=1}^t\mathbbm{1}_{\{\min_a\langle X_{a,s},v\rangle^2 \le h\|v\|_2^2\}}
\end{equation*}
Then $N_t$ is the sum of i.i.d. indicator random variables which have expectation greater than $(1-\zeta Kh)$. Now we define $S_t=\sum_{s=1}^t Z_s$ where
\begin{equation*}
    Z_s=\mathbbm{1}_{\{\min_a\langle X_{a,s},v\rangle^2 \le h\|v\|_2^2\}}-\mathbb{E}[\mathbbm{1}_{\{\min_a\langle X_{a,s},v\rangle^2 \le h\|v\|_2^2\}}],
\end{equation*}
are centered random variables. Then by Hoeffding's Lemma, we have for $\forall t\in \mathbb{N}_+$ and $\delta\in \mathbb{R}$,
\begin{equation*}
    P(\frac{1}{t}S_t\le \delta)\le \exp(-2\delta^2t).
\end{equation*}
Let $\delta=-\frac{1}{2}(1-\zeta Kh)$, we have
\begin{equation*}
    P\left(\frac{1}{t}N_t\le \frac{1}{2}(1-\zeta Kh)\right)\le \exp(-\frac{1}{2}(1-\zeta Kh)^2t)=\exp(-c_0t)
\end{equation*}
where $c_0=\frac{1}{2}(1-\zeta Kh)^2$. Then taking the union of the event over all vectors in $\mathcal{N}_0$, we can obtain
\begin{align}
    & P\left(\forall v\in \mathcal{N}_0: \frac{1}{t}\sum_{s=1}^t \langle X_{a_s,s},v\rangle^2 \le \frac{1}{2}h(1-\zeta  Kh)\right) \nonumber\\
    \le & N\exp(-c_0t) \leq \exp(-c_0t+Cs_0\log d+Cs_0\log(\frac{1}{\epsilon})).\label{eq:12}
\end{align}

In the above inequalities, we have proved the minimum eigenvalue condition for all vectors in $\epsilon$-net $\mathcal{N}_0$. In the following, we will show that the restricted eigenvalue condition for $Cs_0$-sparse vectors can also be implied by the eigenvalue condition on $\mathcal{N}_0$. Before that, we state a result on the spectral norm of symmetric matrices.

\begin{lemma}[Computing the spectral norm on a net]\label{lemma:spectral_net}
Let $A$ be a symmetric $d\times d$ matrix, and let $\mathcal{N}_\epsilon$ be an $\epsilon$-net of $S^{d-1}$ for some $\epsilon\in [0,1)$. The spectral norm of $A$ can be computed via the associated quadratic form $\|A\|=\sup _{x \in S^{d-1}}|\langle A x, x\rangle|$. Then
\begin{equation*}
    \|A\|=\sup _{x \in S^{d-1}}|\langle A x, x\rangle| \leq(1-2 \epsilon)^{-1} \sup _{x \in \mathcal{N}_{\epsilon}}|\langle A x, x\rangle|.
\end{equation*}
\end{lemma}
The proof of Lemma \ref{lemma:spectral_net} can be found in Lemma 5.4 of \citet{vershynin}.

According to the definition of $\mathcal{N}_0$, for any $Cs_0$-sparse vector $v$ and symmetric $d\times d$ matrix $A$, there is a vector $u\in\mathcal{N}_0$, such that $\|u-v\|_2\leq \epsilon$ and $\supp(v)\subseteq\supp(u)$. Then
\begin{align}
    v^T A v&=u^T A u+(u-v)^TA(u-v)+2u^TA(v-u) \nonumber\\
        &\geq u^T A u-2\|u-v\|_2\|A\|_{op} \|u\|_2,\label{eq:A9}
\end{align}
where $\|A\|_{op}$ is the spectral norm of $A$ constrained on sparse vectors with $\ell_0$-norm not greater than $Cs_0$, i.e. $\|A\|_{op}=\underset{v\in S^{d-1},\|v\|_0\leq Cs_0}{\max} \|Av\|_2.$ 

%Then based on \eqref{eq:A9} and the fact that $\|u-v\|_2\leq \epsilon$, we only need to bound the spectral norm $\|A\|_{op}$, which we can prove using Lemma \ref{lemma:spectral_net}.

%\begin{lemma}\label{lemma:7}
%For any matrix $A\in\mathbb{S}_+^d$, $v\in \mathbb{C}_{\alpha}(S)$ and $u\in\mathcal{N}_0$, we have
%\begin{align*}
%        v^T A v&=u^T A u+(u-v)^TA(u-v)+2u^TA(v-u)\\
%        &\ge u^T A u-2\|u-v\|_2\|A\|_{op}\|u\|_2,
%\end{align*}
%where $\|A\|_{op}$ is the $\ell_2$-spectral norm of $A$ on $\mathcal{N}_0$, i.e. $\|A\|_{op}=\max_{v\in\mathcal{N}_0} \|Av\|_2 $.
%\end{lemma}

We have proved a uniform lower bound of $u^T\hat{\Sigma}_tu$ for $u\in \mathcal{N}_0$ and $t\ge T_1$. Then to prove a lower bound of $v^T\hat{\Sigma}_tv$ for $v\in S^{d-1}$ and $\|v\|_0\leq Cs_0$, based on \eqref{eq:A9} and the fact that there exit $u\in \mathcal{N}_0$ such that $\|u-v\|_2\leq \epsilon$, it suffices to show that $\|\hat{\Sigma}_t\|_{op}$ is upper bounded. And by Lemma \ref{lemma:spectral_net}, we only need to bound the spectral norm of $\hat\Sigma_t$ on $\mathcal{N}_0$.

Here we have for a fixed vector $v\in \mathcal{N}_0$ and $\forall t\in\mathbb{N}_+$,
\begin{align*}
    \frac{1}{t}\sum_{s=1}^t\langle X_{a_s,s},v\rangle^2&\le \frac{1}{t}\sum_{s=1}^t\max_{a\in [K]}\langle X_{a,s},v\rangle^2.
\end{align*}

Define $Z_s^a=\langle X_{a,s},v\rangle^2-v^T\Sigma_a v$, $S_t^a=\sum_{s=1}^tZ_s^a$ and $S_t=\sum_{a=1}^K S_t^a$ for $a\in [K]$. Then $Z_s^a$ is centered sub-exponential random variables. Before proving the upper bound of $\|\hat{\Sigma}_t\|_{op}$, we will state some properties of sub-exponential random variables. Firstly, we introduce some parameters of sub-guassian and sub-exponential random variables.

\begin{definition}
 The sub-guassian norm of a sub-guassian random variable $X$, denoted $\|X\|_{\psi_2}$, is defined as
 \begin{equation*}
     \|X\|_{\psi_{2}}=\sup _{p \geq 1} p^{-1 / 2}\left(\mathbb{E}|X|^{p}\right)^{1 / p}.
 \end{equation*}
 Similarly, the sub-exponential norm of a sub-exponential random variable $X$, denoted $\|X\|_{\psi_1}$, is defined as
 \begin{equation*}
     \|X\|_{\psi_{1}}=\sup _{p \geq 1} p^{-1}\left(\mathbb{E}|X|^{p}\right)^{1 / p}.
 \end{equation*}
\end{definition}
The details of the definition and related properties can be found in Section 5.2.3 and Section 5.2.4 in \citet{vershynin}. The following is a relationship between sub-guassian random variables and sub-exponential random variables.

\begin{lemma}\label{lemma:8}
(Sub-exponential is sub-guassian squared). A random variable $X$ is sub-gaussian if and only if $X^2$ is sub-exponential. Moreover, 
\begin{equation*}
    \|X\|_{\psi_2}^2\le \|X^2\|_{\psi_1}\le 2\|X\|_{\psi_2}^2.
\end{equation*}
\end{lemma}

Proof of Lemma \ref{lemma:8} can be found in Lemma 5.14 of \citet{vershynin}.

\begin{lemma}\label{lemma:9}
(MGF of sub-exponential random variables). Let $X$ be a centered sub-exponential random variable. Then for $t$ such that $|t|\le c_1/\|X\|_{\psi_1}$, one has
\begin{equation*}
    \mathbb{E}\exp(tX)\le \exp(c_2t^2\|X\|_{\psi_1}^2),
\end{equation*}
where $c_1,c_2>0$ are absolute constants.
\end{lemma}

Proof of Lemma \ref{lemma:9} can be found in Lemma 5.15 of \citet{vershynin}.

By the result in Lemma \ref{lemma:8}, we have for sub-exponential random variable $Z_s^a$,
\begin{equation*}
    \kappa_a=\|Z_s^a\|_{\psi_1}\le 2\|\langle X_{a,s},v\rangle^2\|_{\psi_1}\le 4\|\langle X_{a,s},v\rangle\|^2_{\psi_2}\le 4\lambda_{\max}^2(\Sigma_a).
\end{equation*}

Then for any $\mu>0$ and $\delta>0$, we have 
\begin{align}
    P\left(\frac{1}{t}\sum_{s=1}^t\max_{a\in [K]}\langle X_{a,s},v\rangle^2-\Lambda_0^2\ge \delta\right)&=P\left(\sum_{s=1}^t[\max_a\langle X_{a,s},v\rangle^2-\Lambda_0^2]\ge \delta t\right) \nonumber\\
    &\le e^{-\mu \delta t}\mathbb{E}[\exp(\mu\sum_{s=1}^t[\max_a\langle X_{a,s},v\rangle^2-\Lambda_0^2] )] \nonumber\\
    &=e^{-\mu \delta t}\prod_{s=1}^t\mathbb{E}[\exp(\mu[\max_a\langle X_{a,s},v\rangle^2-\Lambda_0^2] )]\label{eq:13}
\end{align}

According to Lemma \ref{lemma:9}, for $\mu$ such that $|\mu|\le \frac{c_1}{\max_a \kappa_a}$, we have
\begin{align*}
    \mathbb{E}[\exp(\mu[\max_a\langle X_{a,s},v\rangle^2-\Lambda^2] )]&\le \sum_{a=1}^K \mathbb{E}[\exp\{\mu(\langle X_{a,s},v\rangle^2-\mathbb{E}[\langle X_{a,s},v\rangle^2])\}]\\
    &\le \sum_{a=1}^K e^{c_2\mu^2\kappa^2_a}\le e^{c_2\mu^2\max_a\kappa_a^2+\log(K)}.
\end{align*}

Applying the above inequality to inequality \eqref{eq:13}, we have for $\mu$ such that $|\mu|\le \frac{c_1}{\max_a \kappa_a}$,
\begin{equation}\label{eq:14}
    P\left(\frac{1}{t}\sum_{s=1}^t\max_{a\in[K]}\langle X_{a,s},v\rangle^2-\Lambda_0^2\ge \delta\right)\le \exp(-\mu \delta t+t c_2\mu^2\max_a\kappa_a^2+t\log K).
\end{equation}
The right hand side of the above inequality achieves the minimum value at $\mu=\frac{\delta}{2c_2\max_a\kappa_a^2}$, where the minimum value is $\exp(-\frac{\delta^2t}{4c_2\max_a\kappa_a^2}+t\log K)$.

Now if $\frac{\delta}{2c_2\max_a\kappa_a^2}>\frac{c_1}{\max_a \kappa_a}$, the right hand side of the equation \eqref{eq:14} obtain the minimum value at $\mu=\frac{c_1}{\max_a \kappa_a}$ and 
\begin{align*}
    &P\left(\frac{1}{t}\sum_{s=1}^t\max_{a\in[K]}\langle X_{a,s},v\rangle^2-\Lambda_0^2\ge \delta\right) \\
    \leq & \exp\left(-\frac{c_1\delta t}{\max_a\kappa_a}+tc_2\frac{c_1^2}{\max_a\kappa_a^2}\max_a\kappa_a^2+t\log K\right).\label{eq:15}
\end{align*}

Moreover, if $\frac{\delta}{2c_2\max_a\kappa_a^2}>\frac{c_1}{\max_a\kappa_a}$, we apply this to the above inequality and obtain
\begin{equation*}
    P\left(\frac{1}{t}\sum_{s=1}^t\max_a\langle X_{a,s},v\rangle^2-\Lambda_0^2\ge \delta\right)\le \exp\left(-\frac{c_1\delta t}{2\max_a\kappa_a}+t\log K\right).
\end{equation*}

Now we have for $\forall \delta>0$ and $t\in \mathbb{N}_+$
\begin{align*}
    & P\left(\frac{1}{t}\sum_{s=1}^t\max_a\langle X_{a,s},v\rangle^2-\Lambda_0^2\ge \delta\right) \\
    \leq & \exp\left(-\min\{\frac{\delta^2}{4c_2\max_a\kappa_a^2},\frac{c_1\delta}{2\max_a\kappa_a}\}t+t\log K\right).
\end{align*}

If we set $\delta=c_3\log (K)\Lambda_0^2$ where $c_3$ is a sufficient large positive constant, then we have for a fixed vector $v\in \mathcal{N}_0$
\begin{equation*}
    P\left(\frac{1}{t}\sum_{s=1}^t\max_{a\in[K]}\langle X_{a,s},v\rangle^2-\Lambda_0^2\ge \delta\right) \le \exp\left(-\tilde{C}\log(K)t \right),
\end{equation*}
where $\tilde{C}=\min\{\frac{c_3^4\log K}{64c_2}, \frac{c_1c_3}{8}\}-1\leq \min\{\frac{c_3^2\Lambda_0^4\log K}{4c_2\max_a\kappa_a^2},\frac{c_1c_3\Lambda_0^2}{2\max_a\kappa_a}\} - 1$ is a positive constant independent on $K$ and $t$. Taking the union of the probability over all vectors in $\mathcal{N}_0$, we have
\begin{align}
     & P\left(\forall v\in \mathcal{N}_0: \langle v, \hat{\Sigma}_tv\rangle\ge \tilde{c}\log(K)\Lambda_0^2\right) \nonumber \\
     \leq & N\exp\left(-\tilde{C}\log(K)t \right)\nonumber \\
     \lesssim & \exp\left(-\tilde{C}\log(K)t+Cs_0\log d+Cs_0\log(\frac{1}{\epsilon})\right).\label{eq:A13}
\end{align}
Now we provide the proof of Proposition \ref{Prop:3}.
\begin{proof}
We define events $\mathcal{G}_1$ and $\mathcal{G}_2$ as 
\begin{align*}
    \mathcal{G}_1&=\{\forall u\in \mathcal{N}_0: \frac{1}{t}\sum_{s=1}^t \langle X_{\hat{a}_s,s}, u\rangle \ge \frac{1}{2}h(1-\zeta Kh)\};\\
    \mathcal{G}_2&=\{\forall u\in \mathcal{N}_0: \langle u, \hat{\Sigma}_t u\rangle\le \tilde{c}\log(K)\Lambda_0^2\}.
\end{align*}
If event $\mathcal{G}_2$ holds, then according to Lemma \ref{lemma:spectral_net} and definition of $\|\cdot\|_{op}$ in \eqref{eq:A9}, we have 
\begin{align}
    \|\hat\Sigma_t\|_{op} \leq (1-2\epsilon)^{-1} \max_{v\in \mathcal{N}_0} \langle v, \hat\Sigma_tv\rangle \leq (1-2\epsilon)^{-1} \tilde{c}\log(K)\Lambda_0^2.
\end{align}
If both event $\mathcal{G}_1$ and $\mathcal{G}_2$ hold, based on \eqref{eq:A9}, we can prove the restricted eigenvalue condition for all $Cs_0$-sparse vectors in unit sphere $S^{d-1}$, i.e. for $\forall v\in S^{d-1}$ and $\|v\|_0\leq Cs_0$ that
\begin{align}
    \langle v,\hat\Sigma_t v\rangle & \geq \frac{1}{2}h(1-\zeta Kh) - 2\epsilon \|\hat\Sigma_t\|_{op} \nonumber\\
    &\geq \frac{1}{2}h(1-\zeta Kh) - \frac{2\epsilon}{1-2\epsilon} \tilde{c} \log (K) \Lambda_0^2.
\end{align}
Now if we set $h=\frac{1}{2\zeta K}=\mathcal{O}(\frac{1}{K})$ and $\epsilon=\min\{\frac{1}{4}, \frac{1}{64\tilde{c}\log (K) \Lambda_0^2}\}$, then for $\forall v\in S^{d-1}$ and $\|v\|_0\leq Cs_0$, we have
\begin{equation*}
    \langle v, \hat{\Sigma}_t v\rangle \geq \frac{1}{4}h(1-\zeta  Kh)=\frac{1}{16\zeta  K}.
\end{equation*}

Now we proved part (i) and (ii) in Proposition \ref{Prop:3}, i.e. the minimum eigenvalue condition for all $Cs_0$-sparse vectors in $S^{d-1}$. To finished the proof, we need to apply Lemma \ref{lemma:transfer} to prove part (iii).

To apply Lemma \ref{lemma:transfer}, we set $\Sigma = \frac{1}{8\zeta K}I_{d\times d}$ as a $d\times d$ diagonal matrix, $m=Cs_0$ and $\eta=\frac{1}{2}$. Moreover, we set $D_{j,j}=[\hat\Sigma_t]_{j,j}$ to satisfy the condition on diagonal matrix $D$. By greedy method for constructing $\mathcal{N}_0$, we can let $e_j\in\mathcal{N}_0$ for $\forall j\in[d]$, where $e_j$ is unit vector in $\mathbb{R}^d$ having exactly one entry equal to $1$ and $0$ otherwise. Then under event $\mathcal{G}_2$, we have 
\begin{align*}
    D_{j,j} = [\hat\Sigma_t]_{j,j} = \langle e_j, \hat\Sigma_t e_j\rangle \leq \tilde{c}\log(K)\Lambda_0^2.
\end{align*}

Now according to \eqref{eq:transfer}, if both events $\mathcal{G}_1$ and $\mathcal{G}_2$ hold, we have for $\forall x \in S^{d-1}$ and $x\in\mathbb{C}_3(S)$ that
\begin{align}
    x^T\hat\Sigma_t x & \geq \frac{1}{2} x^T\Sigma x - \frac{\|D^{1/2}x\|^2_1}{Cs_0 - 1} \nonumber\\
    & \geq \frac{1}{16\zeta K} - \frac{\tilde{c}\log(K)\Lambda_0^2 \|x\|_1^2}{Cs_0 - 1}\nonumber\\
    & \geq \frac{1}{16\zeta K} - \frac{\tilde{c}\log(K)\Lambda_0^2 \cdot 16s_0}{Cs_0 - 1},
\end{align}
where the last inequality above is because for $x\in S^{d-1}$ and $x\in\mathbb{C}_3(S)$, we have
$$\|x\|_1\leq (1+3)\|x_S\|_1\leq 4\sqrt{s_0}\|x_S\|_2\leq 4\sqrt{s_0}.$$

Then, if we set $C = 2\cdot 16^2\tilde{c}\log(K)\Lambda_0^2\zeta  K +\frac{1}{s_0}=\Theta(\zeta \Lambda_0^2 K\log K)$, we can have $x^T\hat\Sigma_t x \geq \frac{1}{32\zeta K}$ for unit vector $x\in\mathbb{C}_3(S)$, which implies compatibility condition.

%Now combining the inequalities (\ref{eq:12}) and (\ref{eq:16}), we can show the first inequality in Proposition \ref{Prop:3}. 

Now, we prove the uniform bound on the compatibility condition in Proposition \ref{Prop:3} by showing that events $\mathcal{G}_1$ and $\mathcal{G}_2$ hold with high probability. 

At first, for any $T_1\in \mathbb{N}_+$ we have
\begin{align*}
    &P\left(\forall t\ge T_1,\forall v\in \mathcal{N}_0: \frac{1}{t}\sum_{s=1}^t \langle X_{a_s,s},v\rangle^2\le \frac{1}{2}h(1-\zeta  Kh)\right)\nonumber \nonumber\\
    \leq & N\sum_{t=T_1}^{\infty}e^{-c_0t}= \frac{1}{1-e^{-c_0}}\exp(-c_0T_1+\log N)\\
    \lesssim & \exp(-c_0T_1+Cs_0\log d + Cs_0\log(\frac{1}{\epsilon}))=\delta_2,
\end{align*}
where $c_0=\frac{1}{2}(1-\zeta Kh)^2=\frac{1}{8}$. If we set $\delta_2=\frac{1}{T}$, then with the choice of $\epsilon=\min\{\frac{1}{4}, \frac{1}{64\tilde{c}\log (K) \Lambda_0^2}\}$ and $C$ from above, we have
\begin{align}
    T_1 &=\mathcal{O}(\log T+Cs_0[\log d + \log(\frac{1}{\epsilon})]) \nonumber \\
    &=\mathcal{O}(\log T + s_0\zeta \Lambda_0^2K(\log K)\log d + \log(\log K) + \log\Lambda^2_0) \nonumber \\
    &=\mathcal{O}(\log T + s_0K (\log K)\log d).
\end{align}
The above inequality provides a uniform lower bound for the eigenvalue of sample covariance matrix $\hat{\Sigma}_t$ over $\mathcal{N}_0$ for $t\ge T_1$.

Moreover, we can obtain an union probability bound for all vectors in $\mathcal{N}_0$ and some $T_2\in\mathbb{N}_+$ in \eqref{eq:A13}.
\begin{align*}
    & P\left(\forall t\ge T_2, \forall v\in \mathcal{N}_0: \langle v, \hat{\Sigma}_tv\rangle\geq \tilde{c}\log(K)\Lambda_0^2\right)\\
    \leq & N\sum_{t=T_2}^{\infty} \exp(-\tilde{C}\log(K)t) \\
    \lesssim & \exp\left(-T_2\tilde{C}\log(K)+Cs_0\log d+ Cs_0\log(\frac{1}{\epsilon})\right)=\delta_3.
\end{align*}
If we set $\delta_3=\frac{1}{T}$, then we have
\begin{equation*}
    T_2=\frac{\log(T)+Cs_0\log d+Cs_0\log(\frac{1}{\epsilon})}{\tilde{C}\log(K)}=\mathcal{O}\left(\log T+ Cs_0\log d+Cs_0\log(\frac{1}{\epsilon})\right).
\end{equation*}

Now, we can show that compatibility condition holds on $\mathbb{C}_3(S)$ uniformly for all $t\ge \max(T_1,T_2)=\mathcal{O}(\log T+ s_0K\log K\log d)$ with high probability, i.e.
\begin{equation*}
    P\left(\forall t\ge \max(T_1,T_2), v\in \mathbb{C}_3(S): v^T\hat{\Sigma}_t v\ge \frac{1}{32\zeta K} \right)\ge 1-\delta_2-\delta_3= 1-\frac{2}{T}. %\quad \square
\end{equation*}
\end{proof}

With the crude compatibility condition from Proposition \ref{Prop:3}, we can prove Proposition \ref{Prop:2}.
Firstly, we prove that the optimal arm will be pulled a positive fraction of time by the $\ell_1$-confidence ball based algorithm after some time point. 

\begin{lemma}\label{lemma:10}
Suppose we construct the confidence set $\mathcal{C}_t$ in Proposition \ref{Prop:2} with $\phi=\frac{1}{\sqrt{32\zeta K}}$, then when time horizon $t$ exceeds a certain threshold (i.e. $t\ge\mathcal{O}(\log T+s_0K\log K\log d+\frac{s_0^2\zeta ^2K^2\log d}{\Delta_*^2})$), decision-makers will pull the optimal arms a positive fraction of time.
\end{lemma}

\begin{proof}
Before the proof, we first define the event in Proposition \ref{Prop:1} with $\phi^2=\frac{1}{32\zeta K}$,
\begin{equation*}
  \mathcal{G}_{3,t+1}=\left\{\|\hat{\beta}_t-\beta_\ast\|_1\le  \frac{6s_0 \sigma x_{\max}}{1/(32\zeta K)}\sqrt{\frac{2\log t+2\log d}{t}}\right\}.
\end{equation*}
According to Proposition \ref{Prop:3}, $\mathcal{G}_{3,t}$ will hold uniformly over $t\ge T'$ with high probability.  

From Assumption~\ref{Assump:stochastic}\ref{Assump:eigenvalue}, we have event $\Gamma_t$ such that $P(\Gamma_t)\ge \frac{1}{2}$. Then with sufficient large $\phi\geq \frac{1}{\sqrt{32\zeta K}}$ in Proposition \ref{Prop:2}, if both $\Gamma_t$ and $\mathcal{G}_{3,t}$ hold, we can only select the arm $b\in[K]$ instead of the optimal arm $a_t^*$ when
\begin{align*}
   &  \max_{\beta\in \mathcal{C}_t} \langle X_{b,t}, \beta \rangle \ge  \max_{\beta\in \mathcal{C}_t} \langle X_{a^*_t,t}, \beta \rangle\\
 \Longrightarrow \;  & \langle X_{b,t}, \beta_\ast \rangle +2\|X_{b,t}\|_{\infty}\frac{6s_0 \sigma x_{\max}}{1/(32\zeta K)}\sqrt{\frac{2\log (t-1)+2\log d}{t-1}}\ge \langle X_{a^*_t,t}, \beta_\ast \rangle.\\
\end{align*}
The second inequality is because for all $\beta\in\mathcal{C}_t$, by triangle inequality we have
\begin{align*}
    \|\beta-\beta_\ast\|_1 & \leq \|\beta -\hat\beta_t\|_1 +\|\hat\beta_t-\beta_\ast\|_1 \\
    &\leq (\frac{1}{\phi^2}+32\zeta K)(6s_0 \sigma x_{\max})\sqrt{\frac{2\log (t-1)+2\log d}{t-1}}\nonumber \\
    &\leq 384\zeta Ks_0\sigma x_{\max}\sqrt{\frac{2\log (t-1)+2\log d}{t-1}}
\end{align*}
Then we will always select the optimal arm $a^*_t$ if 
\begin{align*}
    \Delta_*  & \ge 384s_0 \sigma x^2_{\max}\zeta K\sqrt{\frac{2\log (t-1)+2\log d}{t-1}}\\
    & \ge 2\|X_{b,t}\|_{\infty}\frac{6s_0 \sigma x_{\max}}{1/(32\zeta K)}\sqrt{\frac{2\log (t-1)+2\log d}{t-1}}.
\end{align*}
By solving the above inequality, we have that for $t\ge\mathcal{O}(\frac{\sigma^2 x_{\max}^4 s_0^2\zeta ^2K^2\log d}{\Delta_*^2})$, we will always select the optimal arm $a^*_t$ when both $\Gamma_t$ and $\mathcal{G}_{3,t}$ hold.
\end{proof}

\begin{proof}(Proposition~\ref{Prop:2})
Now we consider time such that $t\ge T''=\mathcal{O}(\log T+s_0K \log K\log d+\frac{\sigma^2 x_{\max}^4 s_0^2\zeta ^2K^2\log d}{\Delta_*^2})$. From the above proof, we know that we will select the optimal arm with high probability.

Suppose $t\ge 2T''$ and $v\in \mathbb{R}^d$, then we have
\begin{align}
    \frac{1}{t}\sum_{s=1}^t \langle X_{\hat{a}_s,s},v\rangle^2 & \ge \frac{1}{t} \sum_{s=T''+1}^t \langle X_{\hat{a}_s,s},v\rangle^2 \nonumber\\
    & \geq \frac{1}{t}\sum_{s=T''+1}^t \langle X_{\hat{a}_s,s},v\rangle^2 \mathbbm{1}_{\{\Gamma_s\cap \mathcal{G}_{3,t}\}} \nonumber\\ &=\frac{1}{t}\sum_{s=T''+1}^t \langle X_{a_s^*,s},v\rangle^2 \mathbbm{1}_{\{\Gamma_s\}}\mathbbm{1}_{\{\mathcal{G}_{3,t}\}}\label{eq:17}
\end{align}
Since event $\mathcal{G}_{3,t}$ only depends on the history up to time $t-1$, the random variables $\{X_{a^*_t,t}|\mathcal{G}_{3,t}\}$ are i.i.d. Moreover, since $\{\mathbbm{1}_{\{\Gamma_s\}}: s\ge T''\}$ are i.i.d Bernoulli random variables, then according the Hoeffding's Lemma, we have 
\begin{equation}\label{eq:18}
    \sum_{s=T''+1}^t \mathbbm{1}_{\{\Gamma_s\}} \le \frac{1}{2} \cdot \frac{1}{2}(t-T'') \quad \text{w.p.} \, \le \mathcal{O}(e^{-c_4t}),
\end{equation}
where $c_4$ is some constant.

Now we consider the sum of $\langle X_{a_s^*,s},v\rangle^2$ over times when event $\Gamma_t$ holds
\begin{equation*}
\sum_{s:\Gamma_s \,\text{holds}} \langle X_{a_s^*,s},v\rangle^2-\mathbb{E}[\langle X_{a_s^*,s},v\rangle^2\big| \Gamma_t].
\end{equation*}

Define $Z_s=\langle X_{a_s^*,s},v\rangle^2-\mathbb{E}[\langle X_{a_s^*,s},v\rangle^2\big| \Gamma_t]$, then $Z_s$ are i.i.d. sub-exponential random variables. Moreover, $\mathbb{E}[\langle X_{a_s^*,s},v\rangle^2\big| \Gamma_t]\ge \phi_0^2 \|v\|_2^2$ according to Assumption \ref{Assump:eigenvalue}, and from Lemma \ref{lemma:8}, we have condition on event $\Gamma_t$
\begin{equation*}
        \|Z_s\|_{\psi_1}\le 2\|\langle X_{a_s^*,s},v\rangle^2\|_{\psi_1}\le 4\|\langle X_{a_s^*,s},v\rangle\|_{\psi_2}^2= 4v^T\mathbb{E}[X_{a^*_t}X_{a^*_t}^T|\Gamma_t]v\le 4\Lambda_1^2.
\end{equation*}
Then applying Bernstein inequality for sub-exponential random variables, we can have a lower bound for a fixed vector $v$ and $\epsilon=\frac{\phi^2_0}{2}$.
\begin{align}
    &P\left(|\sum_{s:\Gamma_s \,\text{holds}} Z_s|\ge \epsilon m \bigg| \sum_{s=T''+1}^t \mathbbm{1}_{\{\Gamma_s\}}=m\right) \nonumber \\
    \leq &  2\exp\left(-\tilde{c}_5\min(\frac{\epsilon^2}{\|Z_s\|_{\psi_1}^2}, \frac{\epsilon}{\|Z_s\|_{\psi_1}}) m\right) \nonumber\\
    =& 2\exp\left(-\tilde{c}_5\min(\frac{\phi_0^4}{4\|Z_s\|_{\psi_1}^2}, \frac{\phi_0^2}{2\|Z_s\|_{\psi_1}}) m\right) =\mathcal{O}(e^{-c_5m}),\label{eq:19}
\end{align}
where $\tilde{c}_5, c_5$ are constants.

Now combining inequalities \eqref{eq:17}--\eqref{eq:19}, we have for any $v$ in $\epsilon'$-net $\mathcal{N}_1$ for $S^{d-1}\cap (\cup_{|J|=C's} E_J)$, where $\epsilon$ and $C'$ need to be specified, and $t\ge 2T''$ that 
\begin{align*}
    \frac{1}{t}\sum_{s=1}^t \langle X_{\hat{a}_s,s}, v\rangle^2 &\ge 
    \frac{1}{t}\sum_{s=T''+1}^t \langle X_{a_s^*,s},v\rangle^2 \mathbbm{1}_{\{\Gamma_s\}} \ge  \frac{1}{t}\frac{\phi_0^2}{2}\frac{\frac{1}{2}\cdot(t-T'')}{2}\|v\|_2^2\ge \frac{\phi^2_0}{16},
\end{align*}
with probability greater than $1-\mathcal{O}( e^{-c_4(t-T'')}+e^{-c_5\cdot\frac{1}{4}(t-T'')})$.

Similar as before, we can extend the above inequality to all vectors in $\mathbb{C}_3(S)$ by using the approximation of $\mathcal{N}_1$, where $C=\mathcal{O}(\log K)$ and $\epsilon=\mathcal{O}(\frac{1}{\log K})$. Then for any $v\in \mathbb{C}_3(S)$ and $t\ge 2T''$, we have 
\begin{equation*}
    \hat{\Sigma}_t\in \mathcal{C}(supp(\beta_\ast),\frac{\phi_0}{8})
\end{equation*}
with probability greater than $1-\mathcal{O}(e^{s_0 \log K \log d}[ e^{-c_4(t-T'')}+e^{-c_5\cdot\frac{1}{4}(t-T'')}])=1-\mathcal{O}(e^{-t+s_0\log K \log d})$.

Then combine Proposition \ref{Prop:1} on $\mathcal{G}_{3,t}$ and the above inequality, we can prove the Proposition \ref{Prop:2}. %\quad $\square$
\end{proof}

\subsection{Proof of Theorem \ref{Thm:2}}

Define $T''$ as the threshold in Proposition \ref{Prop:2}, then we divide the cumulative regret into three groups:
\begin{enumerate}[label=(\alph*)]
    \item Initialization when $t\le T''=\mathcal{O}(\log T+s_0K\log K\log d+\frac{\sigma^2 x_{\max}^4 s_0^2\zeta ^2K^2\log d}{\Delta_*^2})$.
    \item Times $t>T''$ when $\hat{\Sigma}_t\notin \mathcal{C}(supp(\beta_\ast),\phi_\ast)$, where $\phi_\ast=\frac{\phi_0}{8}$.
    \item Times $t>T''$ when $\hat{\Sigma}_t\in \mathcal{C}(supp(\beta_\ast),\phi_\ast)$, where $\phi_\ast=\frac{\phi_0}{8}$.
\end{enumerate}

The cumulative regret $J_1$ from time periods in group (a) at time $T$ is bounded by at most $2bx_{\max}T''=\mathcal{O}(2bx_{\max}[\log T+s_0K\log K\log d+\frac{\sigma^2 x_{\max}^4 s_0^2\zeta ^2K^2\log d}{\Delta_*^2}])$.

According to Proposition \ref{Prop:2}, we can bound the cumulative regret in group (b) at time $T$.
\begin{align*}
    J_2=& \sum_{t=T''}^T \mathbb{E}[r_t\mathbbm{1}_{\{\hat{\Sigma}_t\notin \mathcal{C}(supp(\beta_\ast),\phi_\ast)\}}]
    \le \sum_{t=T''}^T 2bx_{\max}P\left(\hat{\Sigma}_t\notin \mathcal{C}(supp(\beta_\ast),\phi_\ast)\right)\\
    \lesssim & \sum_{t=T''}^T 2bx_{\max} \left[\frac{1}{t}+e^{-t+s_0\log k\log d}\right]
    = \mathcal{O}(x_{\max}b\log T).
\end{align*}
To prove the bound of the cumulative regret in group (c), we first define events:
\begin{align*}
    \Gamma_{\xi_t}=\{\omega\in \Omega: \langle X_{a^*_t,t}, \theta_*\rangle\ge \langle X_{b,t},\theta_*\rangle +\xi_t, \; \forall b\neq a^*_t\}.
\end{align*}
In addition, we define $A_t:=\{\hat{\Sigma}_t\in \mathcal{C}(supp(\beta_\ast),\phi_\ast)\}$ for simplicity. Then the cumulative regret at time $T$ in group (c) can be writen as:
\begin{align*}
    J_3 =\sum_{t=T''}^T \mathbb{E}[r_t\mathbbm{1}_{\{A_t\cap \Gamma_{\xi_t} \}}]+\sum_{t=T''}^T \mathbb{E}[r_t\mathbbm{1}_{\{A_t\cap \Gamma_{\xi_t}^c \}}].
\end{align*}
Now we set $\xi_t=2x_{\max}\tau_{t-1}=2x_{\max}\tau_0\sqrt{\frac{\log d+\log (t-1)}{t-1}}$ with $\tau_0=\frac{384\sqrt{2}s_0 \sigma x_{\max}}{\phi_0^2}$, then when both $\Gamma_{\xi_t}$ and $A_t$ hold, we have for any $k\neq a^*_t$ that
\begin{align*}
    & \max_{\beta\in\mathcal{C}_t}\langle X_{a^*_t,t}, \beta\rangle - \max_{\beta\in\mathcal{C}_t}\langle X_{k,t}, \beta\rangle \\
    \ge & \langle X_{a^*_t,t},\beta_\ast\rangle -2x_{\max} \tau_{t-1}-\langle X_{k,t}, \beta_\ast\rangle\ge 0.
\end{align*}
So we will always select the optimal arm $a^*_t$ under the event $A_t\cap \Gamma_{\xi_t}$, and the first part of the cumulative regret in $J_3$ will be zero.

Now we consider the time when both $A_t$ and $\Gamma_{\xi_t}^c$ hold. Then if a sub-optimal arm $k\neq a^*_t$ is selected, the regret incurred will be
\begin{equation*}
     \langle X_{a^*_t,t}-X_{k,t}, \beta_\ast\rangle
    \le  \max_{\beta\in\mathcal{C}_t} \langle X_{a^*_t,t}, \beta\rangle -\max_{\beta\in\mathcal{C}_t} \langle X_{k,t}, \beta\rangle + 2x_{\max}\tau_{t-1}
    \le  2x_{\max}\tau_{t-1} = \xi_t.
\end{equation*}
Moreover, from Assumption~\ref{Assump:reward}\ref{Assump:margin}, we have 
\begin{equation*}
    P\left(\langle X_{a^*_t,t}, \beta_\ast\rangle - \max_{b\neq a^*_t}\langle X_{b,t},\beta_\ast\rangle\le \xi_t\right)\le \frac{1}{2}(\frac{\xi_t}{\Delta_\ast})^{
    \alpha}.
\end{equation*}
(i) For the $\alpha\in [0, 1]$ case, we can bound the second term in $J_3$ as
\begin{align*}
    & \sum_{t=T''}^T \mathbb{E}[r_t\mathbbm{1}_{\{A_t\cap \Gamma_{\xi_t}^c \}}]
    \le  \sum_{t=T''}^T \mathbb{E}[r_t\mathbbm{1}_{\{\text{select arm}\; i\neq s^*_t\}}]\\
    \le & \sum_{t=T''}^T \xi_t P\left(\langle X_{a^*_t,t}, \beta_\ast\rangle - \max_{b\neq a^*_t}\langle X_{b,t},\beta_\ast\rangle\le \xi_t\right)\\
    \le & \sum_{t=T''}^T \frac{1}{2\Delta_\ast^{\alpha}}\xi_t^{\alpha+1} \\
    \le & \sum_{t=1}^T \frac{2^\alpha(x_{\max}\tau_0)^{\alpha+1}}{\Delta_\ast^\alpha} \left(\frac{\log d +\log t}{t}\right)^{\frac{\alpha+1}{2}} \\
    = &
    \left\{\begin{array}{ll}
        C_2\frac{s_0^{\alpha+1}\sigma^{\alpha+1}x_{\max}^{2(\alpha+1)}}{\Delta_\ast^{\alpha}\phi_0^{2(\alpha+1)}}(\log d)^{\frac{\alpha+1}{2}} T^{\frac{1-\alpha}{2}}, & \text{when}~~\alpha\in[0,1), \\
        C_2\frac{s_0^2\sigma^2x_{\max}^4}{\Delta_\ast\phi_0^4}[\log d+\log T]\log T, &  \text{when}~~\alpha=1, %\\
        % C_2\frac{s_0^{\alpha+1}\sigma^{\alpha+1}x_{\max}^{2(\alpha+1)}}{\Delta_\ast^{\alpha}\phi_0^{2(\alpha+1)}}(\log d)^{\frac{\alpha+1}{2}}, & \text{when}~~\alpha\in(1,+\infty),
    \end{array}\right.
\end{align*}
(ii) For the $\alpha\in (1, +\infty)$ case, the second term in $J_3$ can be bounded as
\begin{align*}
    & \sum_{t=T''}^T \mathbb{E}[r_t\mathbbm{1}_{\{A_t\cap \Gamma_{\xi_t}^c \}}] \\
    \leq & \sum_{t=T''}^T \xi_t P\left(\langle X_{a^*_t,t}, \beta_\ast\rangle - \max_{b\neq a^*_t}\langle X_{b,t},\beta_\ast\rangle\le \xi_t\right)\\
    \leq & \sum_{t=T''}^T \xi_t \cdot \min\left(1, \frac{\xi_t^{\alpha}}{2\Delta_{\ast}^{\alpha}}\right).
\end{align*}
Moreover, according to the definition of $\xi_t$, we have that
\begin{align*}
    \frac{\xi_t^{\alpha}}{2\Delta_{\ast}^{\alpha}} \geq 1 
    \iff t \lesssim T''' := \frac{4^{1-1/\alpha}x_{\max}^2\tau_0^2 \log d}{\Delta_{\ast}^2 }.
\end{align*}
Then, the above summation can be decomposed into two parts:
\begin{align}
\label{eq:42}
    & \sum_{t=T''}^T \xi_t \cdot \min\left(1, \frac{\xi_t^{\alpha}}{2\Delta_{\ast}^{\alpha}}\right) \nonumber \\
    \leq & \sum_{t=T''}^{T'''} \xi_t + \sum_{t=T'''}^T \frac{1}{2\Delta_\ast^{\alpha}}\xi_t^{\alpha+1} \nonumber \\
    \leq & x_{\max}bT''' + \frac{2^\alpha(x_{\max}\tau_0)^{\alpha+1}}{\Delta_\ast^\alpha} \sum_{t=T'''}^T \left(\frac{\log d +\log t}{t}\right)^{\frac{\alpha+1}{2}}.
\end{align}
Here for $t\geq T'''$, we have that 
\begin{align*}
    \frac{\log d + \log t}{t} \leq \frac{\log d + \log T'''}{T'''} = \mathcal{O}\left(\frac{\Delta_{\ast}^2}{4^{1-1/\alpha}x_{\max}^2\tau_0^2}\right).
\end{align*}
We define the constant $M:=\frac{\Delta_{\ast}^2}{4^{1-1/\alpha}x_{\max}^2\tau_0^2}$. Then, the second term in the right-hand side of inequality \eqref{eq:42} can be bounded by 
\begin{align*}
    & \sum_{t=T'''}^T \left(\frac{\log d +\log t}{t}\right)^{\frac{\alpha+1}{2}} \\
    \leq & \sum_{t=1}^T \left(\frac{\log d +\log t}{t}\right)^{\frac{\alpha+1}{2}} \mathbbm{1}_{\left\{\frac{\log d + \log t}{t}\leq M\right\}} \\
    \lesssim & \sum_{t=1}^T \left(\frac{\log d}{t}\right)^{\frac{\alpha+1}{2}} \mathbbm{1}_{\left\{\frac{\log d}{t}\leq M\right\}} \\
    \lesssim & \int_1^T \left(\frac{\log d}{t}\right)^{\frac{\alpha+1}{2}} \mathbbm{1}_{\left\{\frac{\log d}{t}\leq M\right\}} ~\mathrm{d}t \\
    = & \int_{\log d/T}^{\min(\log d, M)} \log d \cdot u^{\frac{\alpha+1}{2} - 2}~\mathrm{d}u \\
    \leq & \frac{2}{\alpha-1}\log d \cdot M^{\frac{\alpha-1}{2}} \\
    = & \frac{2\Delta_{\ast}^{\alpha-1}\log d}{(\alpha-1)4^{(\alpha-1)^2/2\alpha} x_{\max}^{\alpha-1}\tau_0^{\alpha-1}}.
\end{align*}
The equality in the above is due to the change of variable by taking $u=\frac{\log d}{t}$. Therefore, the second term in $J_3$ can be bounded for $\alpha \in (1, +\infty)$ as
\begin{align*}
    & \sum_{t=T''}^T \mathbb{E}[r_t\mathbbm{1}_{\{A_t\cap \Gamma_{\xi_t}^c \}}] \\
    \leq & x_{\max}bT''' + \frac{2^\alpha(x_{\max}\tau_0)^{\alpha+1}}{\Delta_\ast^\alpha} \cdot \frac{2\Delta_{\ast}^{\alpha-1}\log d}{(\alpha-1)4^{(\alpha-1)^2/2\alpha} x_{\max}^{\alpha-1}\tau_0^{\alpha-1}}\\
    \leq & x_{\max}bT''' + \frac{2^{3-1/\alpha}x_{\max}^2\tau_0^2\log d}{(\alpha-1)\Delta_{\ast}} \\
    \leq & C_2\left[\frac{bs_0^2\sigma^2 x_{\max}^5\log d}{\Delta_{\ast}^2\phi_0^4} + \frac{s_0^2\sigma^2x_{\max}^4\log d}{(\alpha-1)\Delta_{\ast}\phi_0^4}\right].
\end{align*}
(iii) For the $\alpha=+\infty$ case, Assumption~\ref{Assump:reward}\ref{Assump:margin} implies that for $\xi_t < \Delta_{\ast}$, the following inequality holds.
\begin{align*}
    P\left(\langle X_{a^{\ast}_t,t}, \beta_\ast\rangle - \max_{b\neq a^{\ast}_t}\langle X_{b,t},\beta_\ast\rangle\le \xi_t\right) = 0.
\end{align*}
Therefore, the second term in $J_3$ can be bounded as 
\begin{align*}
   \sum_{t=T''}^T \mathbb{E}[r_t\mathbbm{1}_{\{A_t\cap \Gamma_{\xi_t}^c \}}]
    & \leq  \sum_{t=T''}^T \mathbb{E}[r_t\mathbbm{1}_{\{\text{select arm}\; i\neq s^*_t\}}]\\
    & \leq  \sum_{t=T''}^T \xi_t P\left(\langle X_{a^*_t,t}, \beta_\ast\rangle - \max_{b\neq a^*_t}\langle X_{b,t},\beta_\ast\rangle\le \xi_t\right)\\
    & \leq \sum_{t=T''}^T \xi_t\mathbbm{1}_{\{\xi_t \geq \Delta_\ast\}} \\
    & \overset{(a)}{\leq} \sum_{t=T''}^T \frac{\xi_t^{\gamma+1}}{\Delta_{\ast}^{\gamma}}\\
    & = \sum_{t=T''}^T \frac{2^{\gamma+1}(x_{\max}\tau_0)^{\gamma+1}}{\Delta_\ast^\gamma} \left(\frac{\log d +\log t}{t}\right)^{\frac{\gamma+1}{2}}
\end{align*}
where constant $\gamma \geq 0$. The inequality (a) above is due to $\xi_t/\Delta_{\ast} \geq 1$. By letting $\gamma \to 1_{+}$, we can obtain an upper bound for the $\alpha=\infty $ case, i.e.,
\begin{align*}
    \sum_{t=T''}^T \mathbb{E}[r_t\mathbbm{1}_{\{A_t\cap \Gamma_{\xi_t}^c \}}] \leq C_2\frac{s_0^{2}\sigma^{2}x_{\max}^{4}\log d}{\Delta_\ast \phi_0^{4}}.
\end{align*}
By summing the cumulative regret in three groups, we can obtain the upper bound for the total expected cumulative regret of our $\ell_1$-confidence ball based method up to time $T$.

%%%%%%%%%%%%%%%%%%%%%%%%%%%%%%%%%%%%%%%%%%%%%%
%% Supplementary Material, if any, should   %%
%% be provided in {supplement} environment  %%
%% with title and short description.        %%
%%%%%%%%%%%%%%%%%%%%%%%%%%%%%%%%%%%%%%%%%%%%%%
%\begin{supplement}
%\stitle{???}
%\sdescription{???.}
%\end{supplement}

%% if your bibliography is in bibtex format, uncomment commands:
\bibliographystyle{plainnat} % Style BST file (imsart-number.bst or imsart-nameyear.bst)
\bibliography{refs}       % Bibliography file (usually '*.bib')

\begin{thebibliography}{24}
\providecommand{\natexlab}[1]{#1}
\providecommand{\url}[1]{\texttt{#1}}
\expandafter\ifx\csname urlstyle\endcsname\relax
  \providecommand{\doi}[1]{doi: #1}\else
  \providecommand{\doi}{doi: \begingroup \urlstyle{rm}\Url}\fi

\bibitem[Abbasi-Yadkori et~al.(2011)Abbasi-Yadkori, P\'{a}l, and
  Szepesv\'{a}ri]{abbasi11}
Yasin Abbasi-Yadkori, D\'{a}vid P\'{a}l, and Casaba Szepesv\'{a}ri.
\newblock Improved algorithms for linear stochastic bandits.
\newblock \emph{Advances in Neural Information Processing Systems}, pages
  2312--2320, 2011.

\bibitem[Abbasi-Yadkori et~al.(2012)Abbasi-Yadkori, P\'{a}l, and
  Szepesv\'{a}ri]{abbasi12}
Yasin Abbasi-Yadkori, D\'{a}vid P\'{a}l, and Casaba Szepesv\'{a}ri.
\newblock Online-to-confidence-set conversions and application to sparse
  stochastic bandits.
\newblock \emph{AISTATS}, pages 1--9, 2012.

\bibitem[Abe et~al.(2003)Abe, Biermann, and Long]{abe}
Naoki Abe, Alan~W. Biermann, and Philip~M. Long.
\newblock Reinforcement learning with immediate rewards and linear hypotheses.
\newblock \emph{Algorithmica}, 37\penalty0 (4):\penalty0 236--293, 2003.

\bibitem[Audibert et~al.(2007)Audibert, Tsybakov, et~al.]{audibert2007fast}
Jean-Yves Audibert, Alexandre~B Tsybakov, et~al.
\newblock Fast learning rates for plug-in classifiers.
\newblock \emph{The Annals of statistics}, 35:\penalty0 608--633, 2007.

\bibitem[Auer(2003)]{auer}
Peter Auer.
\newblock Using confidence bounds for exploitation-exploration trade-offs.
\newblock \emph{Journal of Machine Learning Research}, pages 397--422, 2003.

\bibitem[Auer et~al.(2002a)Auer, Cesa-Bianchi, and Fischer]{auer02a}
Peter Auer, Nicol\`{o} Cesa-Bianchi, and Paul Fischer.
\newblock Finite-time analysis of the multiarmed bandit problem.
\newblock \emph{Machine Learning}, 47\penalty0 (2-3):\penalty0 235--256, 2002a.

\bibitem[Bastani and Bayati(2019)]{bastani15}
Hamsa Bastani and Mohsen Bayati.
\newblock Online decision making with high-dimensional covariates.
\newblock \emph{Operations Research}, 68\penalty0 (1):\penalty0 276--294, 2019.

\bibitem[B\"{u}hlmann and Geer(2011)]{highdim}
Peter B\"{u}hlmann and Sara Van~De Geer.
\newblock \emph{Statistics for high-dimensional data: methods, theory and
  applications}.
\newblock Springer Science \& Business Media, 2011.

\bibitem[Candes and Tao(2005)]{candes2005}
Emmanuel Candes and Terence Tao.
\newblock Decoding by linear programming.
\newblock \emph{IEEE Transactions on Information Theory}, 51\penalty0
  (12):\penalty0 4203--4215, 2005.

\bibitem[Chu et~al.(2011)Chu, Li, Reyzin, and Schapire]{chu}
Wei Chu, Lihong Li, Lev Reyzin, and Robert~E Schapire.
\newblock Contextual bandits with linear payoff functions.
\newblock \emph{AISTATS}, pages 208--214, 2011.

\bibitem[Consortium(2009)]{warfarin}
International Warfarin~Pharmacogenetics Consortium.
\newblock Estimation of the warfarin dose with clinical and pharmacogenetic
  data.
\newblock \emph{New England Journal of Medicine}, 360\penalty0 (8):\penalty0
  753, 2009.

\bibitem[Dani et~al.(2008)Dani, Hayes, and Kakade]{dani}
Varsha Dani, Thomas~P. Hayes, and Sham~M. Kakade.
\newblock Stochastic linear optimization under bandit feedback.
\newblock \emph{Conference On Learning Theory}, pages 355--366, 2008.

\bibitem[Fan and Li(2001)]{fan2001variable}
Jianqing Fan and Runze Li.
\newblock Variable selection via nonconcave penalized likelihood and its oracle
  properties.
\newblock \emph{Journal of the American statistical Association}, 96:\penalty0
  1348--1360, 2001.

\bibitem[Goldenshluger and Zeevi(2009)]{golden09}
Alexander Goldenshluger and Assaf Zeevi.
\newblock {Woodroofe’s one-armed bandit problem revisited}.
\newblock \emph{The Annals of Applied Probability}, 19\penalty0 (4):\penalty0
  1603 -- 1633, 2009.

\bibitem[Goldenshluger and Zeevi(2013)]{golden}
Alexander Goldenshluger and Assaf Zeevi.
\newblock {A linear response bandit problem}.
\newblock \emph{Stochastic Systems}, 3\penalty0 (1):\penalty0 230 -- 261, 2013.

\bibitem[Kim and Paik(2019)]{kim2019}
Gi-Soo Kim and Myunghee~Cho Paik.
\newblock Doubly-robust lasso bandit.
\newblock In H.~Wallach, H.~Larochelle, A.~Beygelzimer, F.~d\textquotesingle
  Alch\'{e}-Buc, E.~Fox, and R.~Garnett, editors, \emph{Advances in Neural
  Information Processing Systems}, volume~32. Curran Associates, Inc., 2019.

\bibitem[Oliveira(2016)]{oliveira}
Roberto~Imbuzeiro Oliveira.
\newblock The lower tail of random quadratic forms, with applications to
  ordinary least squares and restricted eigenvalue properties.
\newblock \emph{Probability Theory and Related Fields}, 166\penalty0
  (3):\penalty0 1175--1194, 2016.

\bibitem[Rudelson and Zhou(2013)]{Rudelson}
Mark Rudelson and Shuheng Zhou.
\newblock Reconstruction from anisotropic random measurements.
\newblock \emph{IEEE Transactions on Information Theory}, 59\penalty0
  (6):\penalty0 3434--3447, 2013.

\bibitem[Rusmevichientong and Tsitsiklis(2010)]{paat2010}
Paat Rusmevichientong and John~N. Tsitsiklis.
\newblock Linearly parameterized bandits.
\newblock \emph{Mathematics of Operations Research}, 35\penalty0 (2):\penalty0
  395--411, 2010.

\bibitem[Tibshirani(1996)]{lasso}
Robert Tibshirani.
\newblock Regression shrinkage and selection via the lasso.
\newblock \emph{Journal of the Royal Statistical Society. Series B
  (Methodological)}, pages 267--288, 1996.

\bibitem[Vershynin(2012)]{vershynin}
Roman Vershynin.
\newblock Introduction to the non-asymptotic analysis of random matrices.
\newblock In Yonina~C. Eldar and GittaEditors Kutyniok, editors,
  \emph{Compressed Sensing: Theory and Applications}, pages 210--268. Cambridge
  University Press, 2012.

\bibitem[Wainwright(2019)]{wainwright}
Martin~J. Wainwright.
\newblock \emph{High-Dimensional Statistics: A Non-Asymptotic Viewpoint}.
\newblock Cambridge University Press, 2019.

\bibitem[Wang et~al.(2018)Wang, Wei, and Yao]{wang18}
Xue Wang, Mingcheng Wei, and Tao Yao.
\newblock Minimax concave penalized multi-armed bandit model with
  high-dimensional covariates.
\newblock In \emph{Proceedings of the 35th International Conference on Machine
  Learning}, volume~80 of \emph{Proceedings of Machine Learning Research},
  pages 5200--5208. PMLR, 10--15 Jul 2018.

\bibitem[Zhang et~al.(2010)]{zhang2010nearly}
Cun-Hui Zhang et~al.
\newblock Nearly unbiased variable selection under minimax concave penalty.
\newblock \emph{The Annals of statistics}, 38:\penalty0 894--942, 2010.

\end{thebibliography}

%% or include bibliography directly:
% \begin{thebibliography}{}
% \bibitem{b1}
% \end{thebibliography}

\end{document}